\documentclass[twoside]{article}

\usepackage[accepted]{aistats2019}
\usepackage[utf8]{inputenc} 
\usepackage{url}            
\usepackage{booktabs}       
\usepackage{amsfonts}       
\usepackage{nicefrac}       
\usepackage{microtype}      
\usepackage[round]{natbib}

\usepackage{algorithm}
\usepackage{algorithmic}

\usepackage{amsmath}
\usepackage{amssymb}
\usepackage{amsthm}
\usepackage{amsfonts} 
\usepackage{bm}
\usepackage{setspace}
\usepackage{xspace}
\usepackage{xcolor} 
\usepackage{wrapfig}
\usepackage{array} 
\usepackage{multirow}
\usepackage{enumitem}
\usepackage{mdframed}

\usepackage{thm-restate}
\usepackage{caption}
\usepackage{subcaption}

\usepackage[colorinlistoftodos,disable, textwidth=20mm]{todonotes}

\definecolor{blued}{RGB}{70,197,221}
\definecolor{pearOne}{HTML}{2C3E50}
\definecolor{pearTwo}{HTML}{A9CF54}
\definecolor{pearTwoT}{HTML}{C2895B}
\definecolor{pearThree}{HTML}{E74C3C}
\colorlet{titleTh}{pearOne}
\colorlet{bull}{pearTwo}
\definecolor{pearcomp}{HTML}{B97E29}
\definecolor{pearFour}{HTML}{588F27}
\definecolor{pearFith}{HTML}{ECF0F1}
\definecolor{pearDark}{HTML}{2980B9}
\definecolor{pearDarker}{HTML}{1D2DEC}
\usepackage{hyperref}       
\hypersetup{
	colorlinks,
	citecolor=pearDark,
	linkcolor=pearThree,
	breaklinks=true,
	urlcolor=pearDarker}

\definecolor{citrine}{rgb}{0.89, 0.82, 0.04}

\definecolor{graphicbackground}{rgb}{0.96,0.96,0.8}
\definecolor{rouge1}{RGB}{226,0,38}  
\definecolor{orange1}{RGB}{243,154,38}  
\definecolor{jaune}{RGB}{254,205,27}  
\definecolor{blanc}{RGB}{255,255,255} 
\definecolor{rouge2}{RGB}{230,68,57}  
\definecolor{orange2}{RGB}{236,117,40}  
\definecolor{taupe}{RGB}{134,113,127} 
\definecolor{gris}{RGB}{91,94,111} 
\definecolor{bleu1}{RGB}{38,109,131} 
\definecolor{bleu2}{RGB}{28,50,114} 
\definecolor{vert1}{RGB}{133,146,66} 
\definecolor{vert3}{RGB}{20,200,66} 
\definecolor{vert2}{RGB}{157,193,7} 
\definecolor{darkyellow}{RGB}{233,165,0}  
\definecolor{lightgray}{rgb}{0.9,0.9,0.9}
\definecolor{darkgray}{rgb}{0.6,0.6,0.6}
\definecolor{babyblue}{rgb}{0.54, 0.81, 0.94}\definecolor{citrine}{rgb}{0.89, 0.82, 0.04}
\definecolor{misogreen}{rgb}{0.25,0.6,0.0}

\DeclareMathOperator*{\argmax}{arg\,max}

\newcommand{\EE}[1]{\mathbb{E}\left[#1\right]}

\newcommand{\pa}[1]{\left(#1\right)}

\let\originalleft\left
\let\originalright\right
\renewcommand{\left}{\mathopen{}\mathclose\bgroup\originalleft}
\renewcommand{\right}{\aftergroup\egroup\originalright}

\newcommand{\abs}[1]{\left|#1\right|}

\newcommand{\CommaBin}{\mathbin{\raisebox{0.5ex}{,}}}

\newcommand{\cO}{\mathcal{O}}
\newcommand{\tcO}{\widetilde{\cO}}

\newcommand{\eps}{\varepsilon}
\renewcommand{\epsilon}{\varepsilon}
\renewcommand{\hat}{\widehat}

\renewcommand{\bar}{\overline}

\newcommand{\nothere}[1]{}

\usepackage{xspace}

\newcommand{\UCB}{{\small \textsc{UCB}}\xspace}
\newcommand{\UCBone}{{\small \textsc{UCB1}}\xspace}
\newcommand{\MOSS}{{\small\textsc{MOSS}}\xspace}

\newcommand{\EXP}{{\textsc{Exp3}}\xspace}

\newcommand{\rounds}{{\textcolor[rgb]{0.25,0.0,0.6}{T}}}

\newcommand{\regret}{R_\rounds}

\newcommand{\noise}{\eps}

\newcommand{\narms}{K}
\newcommand{\currentTime}{t}
\newcommand{\lipschitz}{L}
\newcommand{\timeEnd}{T}

\newcommand{\subgaussian}{\sigma}
\newcommand{\arm}{i}
\newcommand{\armCount}{N}
\newcommand{\policy}{\pi}
\newcommand{\reward}{\mu}
\newcommand{\obs}{r}
\newcommand{\totalReward}{J}
\newcommand{\expectation}{\mathop{\mathbb{E}}}
\newcommand{\possibleArms}{\mathcal{K}}
\newcommand{\window}{h}
\newcommand{\estReward}{\hat{\reward}}
\newcommand{\expestReward}{\bar{\reward}}
\newcommand{\historyt}{\mathcal{H}_\currentTime}

\newcommand{\EWA}{\policy_{\rm F}}
\newcommand{\underpullSet}{\textsc{up}}
\newcommand{\overpullSet}{\textsc{op}}

\newcommand{\HPevent}{\xi}

\newcommand{\rewardSet}{\mathcal{L}_\lipschitz}
\newcommand{\myAlgorithm}{{\small \textsc{FEWA}}\xspace}
\newcommand{\FEWA}{{\small \textsc{FEWA}}\xspace}
\newcommand{\SWA}{\normalfont{\textsc{wSWA}}\xspace}
\newcommand{\DUCB}{{\small \textsc{D-UCB}}\xspace}
\newcommand{\SWUCB}{{\small \textsc{SW-UCB}}\xspace}

\newcommand{\EFF}{{\small \textsc{EFF-FEWA}}\xspace}
\newcommand{\EFFP}{\policy_{\rm EF}}
\newcommand{\wb}[1]{\overline{#1}}

\newtheorem{corollary}{Corollary}
\newtheorem{lemma}{Lemma}
\newtheorem{proposition}{Proposition}

\newtheorem{assumption}{Assumption}

\newtheorem{remark}{Remark}

\usepackage{authblk}

\AfterEndEnvironment{assumption}{\noindent\ignorespaces}
\AfterEndEnvironment{proposition}{\noindent\ignorespaces}
\AfterEndEnvironment{theorem}{\noindent\ignorespaces}
\AfterEndEnvironment{restatable}{\noindent\ignorespaces}
\AfterEndEnvironment{lemma}{\noindent\ignorespaces}
\AfterEndEnvironment{remark}{\noindent\ignorespaces}

\begin{document}
	
	\twocolumn[
	
	\aistatstitle{Rotting bandits are not harder than stochastic ones}
	
	\aistatsauthor{Julien Seznec \And  Andrea Locatelli \And Alexandra Carpentier \And Alessandro Lazaric \And Michal Valko}
	
	\aistatsaddress{Lelivrescolaire.fr \And OvGU Magdeburg  \And  OvGU Magdeburg   \And FAIR Paris \And Inria Lille}]


\begin{abstract}\noindent
In stochastic multi-armed bandits, the reward distribution of each arm is assumed to be stationary. This assumption is often violated in practice (e.g., in recommendation systems), where the reward of an arm may change whenever is selected, i.e., rested bandit setting. In this paper, we consider the \textit{non-parametric rotting bandit} setting, where rewards can only decrease. We introduce the \textit{filtering on expanding window average} (\myAlgorithm) algorithm that constructs moving averages of increasing windows to identify arms that are more likely to return high rewards when pulled once more. We prove that for an unknown horizon $T$, and without any knowledge on the decreasing behavior of the $K$ arms, \myAlgorithm achieves problem-dependent regret bound of $\tcO(\log{(KT)}),$ and a problem-independent one of $\tcO(\sqrt{KT})$. Our result substantially improves over the algorithm of~\citet{levine2017rotting}, which suffers regret $\tcO(K^{1/3}T^{2/3})$. \myAlgorithm also matches known bounds for the stochastic bandit setting, thus showing that the rotting bandits are not harder. Finally, we report simulations confirming the theoretical improvements of \myAlgorithm.
\end{abstract}


\vspace{-0.1in}
\section{Introduction}
\label{Introduction}
\vspace{-0.1in}
The multi-arm bandit framework~\citep{bubeck2012regret,lattimore2019bandit} formalizes the exploration-exploitation dilemma in online learning, where an agent has to trade off the \textit{exploration} of the environment to gather information and the \textit{exploitation} of the current knowledge to maximize reward. In the \textit{stochastic setting} \citep{thompson1933likelihood,auer2002finite}, each arm is characterized by a stationary reward distribution. Whenever an arm is pulled, an i.i.d.\,sample from the corresponding distribution is observed. Despite the extensive algorithmic and theoretical study of this setting, 
the stationarity assumption is often too restrictive in practice, e.g., the preferences of users may change over time. The \textit{adversarial setting} \citep{auer2002nonstochastic} addresses this limitation by removing any assumption on how the rewards are generated and learning agents should be able to perform well for any \textit{arbitrary} sequence of rewards. While algorithms such as \EXP \citep{auer2002nonstochastic} are guaranteed to achieve small regret in this setting, their behavior is conservative as all arms are repeatedly explored to avoid incurring too much regret because of unexpected changes in arms' values. This behavior results in unsatisfactory performance in practice, where arms' values, while non-stationary, are far from being adversarial. \citet{garivier2011upper-confidence-bound} proposed a variation of the stochastic setting, where the distribution of each arm is \textit{piecewise stationary}. Similarly,~\citet{besbes2014optimal} introduced an adversarial setting where the total amount of change in arms' values is bounded. These settings fall into the \textit{restless} bandit scenario, where the arms' value evolves \textit{independently} from the decisions of the agent. 
On the other hand, for \textit{rested} bandits, the value of an arm changes only when it is pulled. For instance, the value of a service may deteriorate only when it is actually used, e.g., if a recommender system shows always the same item to the users, they may get bored~\citep{warlop2018fighting}. Similarly, a student can master a frequently taught topic in an intelligent tutoring system and extra learning on that topic would be less effective. A particularly interesting case is represented by the \textit{rotting bandits}, where the value of an arm may decrease whenever pulled. 
\citet{heidari2016tight} studied this problem when rewards are deterministic (i.e., no noise) and showed how a greedy policy (i.e., selecting the arm that returned the largest reward the last time it was pulled) is optimal up to a small constant factor depending on the number of arms $\narms$ and the largest per-round decay in the arms' value $\lipschitz$. \citet{bouneffouf2016multi-armed} considered the stochastic setting when the dynamics of the rewards is known up to a constant factor. Finally, \citet{levine2017rotting} considered both non-parametric and parametric noisy rotting bandits, for which they derive algorithms with regret guarantees. In the non-parametric case, where the decrease in reward is neither constrained nor known, they introduce the \textit{sliding-window average} (\SWA) algorithm, which is shown to achieve a regret to the optimal policy of order $\widetilde{\cO}(K^{1/3}\timeEnd^{2/3})$, where $\timeEnd$ is the number of rounds in the experiment. 

In this paper, we study the non-parametric rotting setting of~\citet{levine2017rotting} and introduce 
\textit{Filtering on Expanding Window Average} (\myAlgorithm) algorithm, a novel method that constructs moving average estimates of increasing windows to identify the arms that are more likely to perform well if pulled once more. Under the assumption that the reward decays are bounded, we show that \myAlgorithm achieves a regret of $\widetilde{\cO}(\sqrt{K\timeEnd})$, thus \emph{significantly improving over \SWA} and matching the minimax rate of stochastic bandits up to a logarithmic factor. This shows that learning with non-increasing rewards is not more difficult than in the stationary case. Furthermore, when rewards are constant, we recover \emph{standard problem-dependent regret guarantees} (up to constants), while in the rotting bandit scenario with no noise, the regret reduces to the one of~\citet{heidari2016tight}. Numerical simulations confirm our theoretical results and show the superiority of \myAlgorithm over \SWA. 

\vspace{-0.1in}
\section{Preliminaries}
\label{Model}
\vspace{-0.1in}

We consider a rotting bandit scenario similar to the one of~\citet{levine2017rotting}. At each round $t$, an agent chooses an arm $\arm(\currentTime) \in \possibleArms \triangleq \left\{ 1, ... , \narms\right\} $ and receives a noisy reward $\obs_{\arm(\currentTime),\currentTime}$. The reward associated to each arm $i$ is a $\subgaussian^2$-sub-Gaussian r.v.\,with expected value of $\mu_i(n)$, which depends on the number of times $n$ it was pulled before; $\mu_i(0)$ is the initial expected value.\footnote{Our definition slightly differs from the one of~\citet{levine2017rotting}. We use $\mu_i(n)$ for the expected value of arm~$i$ \textit{after $n$ pulls} instead of when it is pulled \textit{for the $n$-th time}. 
}  Let $\historyt \triangleq \left\{ \left\{ \arm(s), \obs_{\arm(s), s} \right\}, \forall s < \currentTime \right\}$ be the sequence of arms pulled and rewards observed until round $t$, then 
\begin{align*}
\obs_{\arm(\currentTime),\currentTime} \triangleq \reward_{\arm(\currentTime)}(\armCount_{\arm(\currentTime),\currentTime}) + \noise_\currentTime 
 \;\;\; \text{with}\; \expectation \left[  \noise_\currentTime | \historyt \right] = 0\\
  \text{and} \; \forall \lambda \in \mathop{\mathbb{R}}, \quad \expectation\left[ e^{\lambda\noise_\currentTime}\right] \leq e^{\frac{\subgaussian\lambda^2}{2}},
\end{align*}
where $\armCount_{\arm,\currentTime}\triangleq \sum_{s=1}^{t-1} \mathbb{I}\{i(t) = i\}$ is the number of times arm $i$ is pulled before round $t$. We use $r_i(n)$ to denote the random reward of arm $i$ when pulled for the $n+1$-th time, i.e., $r_{i(t),t} = \obs_{\arm(\currentTime)}(\armCount_{\arm(\currentTime),\currentTime})$. 
We introduce a non-parametric rotting assumption with bounded decay.
\begin{assumption}\label{assum-Lipschitz}
The reward functions $\reward_\arm$ are non-increasing with bounded decays $-\lipschitz \leq \reward_\arm(n+1) - \reward_\arm(n) \leq 0.$ The initial expected value is bounded as $\reward_\arm(0) \in \left[0,\lipschitz\right]$. We refer to this set of functions as $\rewardSet.$
\end{assumption}%

\paragraph{The learning problem} A learning policy~$\pi$ is a function from the history of observations to arms, i.e., $\pi(\mathcal{H}_t) \in \mathcal{K}$. In the following, we often use $\pi(t) \triangleq \pi(\mathcal{H}_t)$. The performance of a policy $\pi$ is measured by the (expected) rewards accumulated over time, 
%
\begin{equation*}
\totalReward_\timeEnd(\policy) \triangleq \sum_{t=1}^\timeEnd  \reward_{\policy(t)}\pa{\armCount_{\policy(t),t}}.
\end{equation*}
Since $\pi$ depends on the (random) history observed over time, $\totalReward_\timeEnd(\policy)$ is also random. We define the expected cumulative reward as $\wb{\totalReward}_T(\policy) \triangleq \mathbb{E}\big[\totalReward_\timeEnd(\policy)\big]$. We now restate a characterization of the optimal (oracle) policy.
\begin{proposition}[\citealp{heidari2016tight}]
If the expected value of each arm $\{\mu_i(n)\}_{i,n}$ is known, the policy $\pi^\star$ maximizing the expected cumulative reward $\overline{\totalReward_T}(\policy)$ is greedy at each round, i.e., 
\begin{align}\label{eq:optimal.policy}
\pi^\star(t) = \arg\max_\arm \reward_\arm\pa{\armCount_{\arm,\currentTime}}.
\end{align}
We denote by $J^\star \triangleq \wb{\totalReward}_T(\policy^\star) = \totalReward_T(\policy^\star)$, the cumulative reward of the optimal policy.
\end{proposition}%
The objective of a learning algorithm is to implement a policy $\pi$ with performance as close to $\pi^\star$'s as possible. We define the (random) regret as 
\begin{align}\label{eq:regret}
\regret(\policy) \triangleq J^\star - \totalReward_T(\policy).
\end{align}
%
Notice that the regret is measured against an optimal allocation over arms rather than a fixed-arm policy as it is a case in adversarial and stochastic bandits. Therefore, even the adversarial algorithms that one could think of applying in our setting (e.g., \EXP of \citealp{auer2002finite}) are not known to provide any guarantee for our definition of regret. On the other hand, for constant $\mu_i(n)$-s, our problem and definition of regret reduce to the one of standard stochastic bandits. 


Let $N_{i,T}^\star$ be the (deterministic) number of rounds that arm~$i$ is pulled by the oracle policy $\pi^\star$ up to round~$T$ (excluded). Similarly, for a policy $\pi$, let $N_{i,T}^\pi$ be the (random) number pulls of arm $i$. The cumulative reward can be rewritten as
\begin{align*}
\totalReward_T(\policy) &= \sum_{t=1}^\timeEnd \sum_{i\in\mathcal{K}} \mathbb{I}_{\{\pi(t) = i\}} \reward_{i}\pa{\armCount_{i,t}^{\policy}} 
= 
\sum_{i\in\mathcal{K}}\sum_{s=0}^{\armCount^\policy_{i,T}-1}  \reward_{i}(s).
\end{align*}
Then, we can conveniently rewrite the regret as
%
\begin{align}
\!\regret(\policy) &= \sum_{i\in\mathcal{K}}\left( \sum_{s=0}^{\armCount_{i,T}^\star-1}  \reward_{i}(s)  - \sum_{s=0}^{\armCount_{i,T}^\pi-1}  \reward_{i}(s) \right) \nonumber\\ 
& = \sum_{\arm\in \underpullSet}\sum_{s=\armCount_{\arm, \timeEnd}^{\policy}}^{\armCount_{\arm, \timeEnd}^{\star}-1} \reward_\arm(s) - \sum_{\arm\in \overpullSet} \sum_{s=\armCount_{\arm, \timeEnd}^{\star}}^{\armCount_{\arm, \timeEnd}^{\policy}-1} \reward_\arm(s),\label{eq:regret2}
\end{align}
%
where we define $\underpullSet \triangleq \left\{ \arm \in \mathcal{K} | \armCount_{\arm, \timeEnd}^{\star} > \armCount_{\arm, \timeEnd}^{\policy} \right\}$ and likewise $\overpullSet \triangleq \left\{ \arm \in \mathcal{K} | \armCount_{\arm, \timeEnd}^{\star} < \armCount_{\arm, \timeEnd}^{\policy}\right\}$ as the sets of arms that are respectively under-pulled and over-pulled by~$\pi$ w.r.t.\,the optimal policy. 

\paragraph{Known regret bounds} We report existing regret bounds for two special cases. We start with the minimax regret lower bound for stochastic bandits.
%
\begin{proposition}{\cite[Thm.\,5.1]{auer2002nonstochastic}}
\label{stochastic-LB}
For any learning policy $\policy$ and any horizon $T$, there exists a stochastic stationary problem $\left\{ \mu_i (n) \triangleq \mu_i\right\}_i$ with $K$ $\sigma$-sub-Gaussian arms such that $\pi$ suffers a regret
\begin{equation*}
 \mathbb{E}[\regret(\policy)] \geq \frac{\sigma}{10}\min\pa{\sqrt{\narms\timeEnd},\timeEnd}.
\end{equation*}
where the expectation is w.r.t.\ both the randomization
over rewards and algorithm's internal randomization.
\end{proposition}
\citet{heidari2016tight} derived regret lower and upper bounds for deterministic rotting bandits (i.e., $\sigma=0$).


\begin{proposition}{\citep[Thm.\,3]{heidari2016tight}}
\label{deterministic-LB}
For any learning policy~$\policy$, there exists a deterministic rotting bandits (i.e., $\sigma=0$) satisfying Assumption~\ref{assum-Lipschitz} with bounded decay~$L$ such that $\policy$ suffers an expected regret~
\begin{equation*}
 \mathbb{E}[\regret(\policy)] \geq \frac{\lipschitz}{2}(\narms-1).
\end{equation*}
Let $\policy^{\subgaussian_0}$ be the greedy policy that selects at each round the arm with the largest reward observed so far, i.e., $\policy^{\subgaussian_0}(t) \triangleq \argmax_\arm(\reward_\arm(\armCount_{\arm,\currentTime}-1))$. For any deterministic rotting bandits (i.e., $\sigma=0$) satisfying Assumption~\ref{assum-Lipschitz} with bounded decay $L$, $\policy^{\subgaussian_0}$ suffers an expected regret
\begin{equation*}
 \mathbb{E}[\regret(\policy^{\subgaussian_0})] \leq \lipschitz(\narms-1).
\end{equation*}
\end{proposition}
Any problem in the two settings above is a rotting problem with parameters ($\sigma$, $L$). Therefore, the performance of any algorithm on the general rotting problem is also bounded by these two lower bounds.

\vspace{-0.1in}
\section{{\myAlgorithm}: Filtering on expanding window average}
\vspace{-0.1in}

Since the expected rewards $\mu_i$ change over time, the main difficulty in the non-parametric rotting bandits  is that we cannot rely on all samples observed until round~$t$ to predict which arm is likely to return the highest reward in the future. In fact, the older a sample, the less representative it is for future rewards. This suggests constructing estimates using the more recent samples. Nonetheless, discarding older rewards reduces the number of samples used in the estimates, thus increasing their variance. In Alg.\,\ref{EWA} we introduce \myAlgorithm (or~$\EWA$) that at each round $t$, relies on estimates using windows of increasing length to filter out arms that are suboptimal with high probability and then pulls the least pulled arm among the remaining arms. 

We first describe the subroutine {\small\textsc{Filter}} in Alg.\,\ref{filter}, which receives a set of active arms $\mathcal{K}_h$, a window~$h$, and a confidence parameter $\delta$ as input and returns an updated set of arms $\mathcal{K}_{h+1}$. For each arm~$i$ that has been pulled~$n$ times, the algorithm constructs an estimate $\estReward^\window_\arm(n)$ that averages the $h \leq n$ most recent rewards observed from~$i$. 
The subroutine {\small\textsc{Filter}} discards all the arms whose mean estimate (built with window~$h$) from $\mathcal{K}_h$  is lower than the empirically best arm by more than twice a threshold $c(\window, \delta_\currentTime)$ constructed by standard Hoeffding's concentration inequality (see Prop.\,\ref{prop:heoffding}). 

\begin{algorithm}[t]
\caption{\myAlgorithm}
\label{EWA}
\begin{algorithmic}[1]
\REQUIRE $\subgaussian$, $\possibleArms$, $\alpha$, $\delta_0 \gets 1$
	\STATE pull each arm once, collect reward, and initialize $N_{\arm,K} \leftarrow 1$ 
	\FOR{$\currentTime \gets K+1, K+2, \dots \do $}
		\STATE $\delta_t \leftarrow \delta_0/(t^\alpha)$
		\STATE $\window \leftarrow 1$ 
		{\footnotesize \COMMENT{\emph{initialize bandwidth}}}
		\STATE $\possibleArms_1 \leftarrow \possibleArms$ 
		{\footnotesize \COMMENT{\emph{initialize with all the arms}}}
		\STATE $\arm(t) \gets {\tt none}$
		\WHILE{$\arm(t)$ is  ${\tt none}$}
			\STATE $\possibleArms_{\window+1} \leftarrow {\textsc{Filter}}(\possibleArms_{\window} ,\window, \delta_\currentTime)$
			\STATE $\window \leftarrow \window+1$ \label{algline:window}
			\IF{$\exists \arm \in \possibleArms_{\window}$ such that $\armCount_{\arm, t}=h$}
			\label{algline:condition}
			\STATE $\arm(t) \leftarrow \arg\min_{i\in\possibleArms_{\window}} N_{i,t}$
			\ENDIF
		\ENDWHILE
		\STATE  receive $\obs_\arm(\armCount_{\arm,\currentTime +1 }) \leftarrow \obs_{\arm(\currentTime),\currentTime}$
		\STATE $\armCount_{\arm(\currentTime),\currentTime} \leftarrow \armCount_{\arm(\currentTime),\currentTime-1} +1$
		\STATE $\armCount_{j,\currentTime} \leftarrow \armCount_{j, \currentTime-1}, \quad \forall j \neq \arm(\currentTime)$
	\ENDFOR
\end{algorithmic}

\end{algorithm}

%
\begin{algorithm}[t]
\caption{{\textsc{Filter}}}
\label{filter}
\begin{algorithmic}[1]
\REQUIRE $\possibleArms_{\window}$, $\window$, $\delta_\currentTime$
\STATE $c(\window, \delta_\currentTime) \leftarrow \sqrt{(2\subgaussian^2/\window) \log{(1/\delta_\currentTime)}}$
\FOR{$ \arm \in \possibleArms_{\window}$}
\STATE $\estReward^\window_\arm(\armCount_{\arm,\currentTime}) \leftarrow \frac{1}{\window} \sum_{j=1}^\window \obs_\arm(\armCount_{\arm,\currentTime} -j)$
\ENDFOR
\STATE $\estReward^\window_{\max,\currentTime}  \leftarrow \max_{\arm \in \possibleArms_{\window}}\estReward^\window_\arm(\armCount_{\arm,\currentTime})$
\FOR{$ \arm \in \possibleArms_{\window}$}
	\STATE $\Delta_\arm \leftarrow  \estReward^\window_{\max,\currentTime}  - \estReward^\window_\arm(\armCount_{\arm,\currentTime})$
	\IF{$\Delta_\arm \leq 2c(\window, \delta_\currentTime) $}
	\STATE add $\arm$ to $\possibleArms_{\window+1}$
	\ENDIF
\ENDFOR
\ENSURE $\possibleArms_{\window+1}$
\end{algorithmic}
\end{algorithm}

The {\small\textsc{Filter}} subroutine is used in \myAlgorithm to incrementally refine the set of active arms, starting with a window of size $1$, until the condition at Line~\ref{algline:condition} is met. As a result, $\mathcal{K}_{h+1}$ only contains arms that passed the filter for all windows from $1$ up to $h$. Notice that it is important to start filtering arms from a small window and to keep refining the previous set of active arms. 
In fact, the estimates constructed using a small window use recent rewards, which are closer to the future value of an arm. As a result, if there is enough evidence that an arm is suboptimal already at a small window $h$, it should be directly discarded. On the other hand, a suboptimal arm may pass the filter for small windows as the threshold $c(\window, \delta_\currentTime)$ is large for small $h$ (i.e., as few samples are used in constructing $\estReward^\window_\arm(\armCount_{\arm,\currentTime})$, the estimation error may be high). Thus, \myAlgorithm keeps refining $\mathcal{K}_{h}$ for larger windows in the attempt of constructing more accurate estimates and discard more suboptimal arms. This process stops when we reach a window as large as the number of samples for at least one arm in the active set $\mathcal{K}_{h}$ (i.e., Line~\ref{algline:condition}). At this point, increasing $h$ would not bring any additional evidence that could refine $\mathcal{K}_{h}$ further (recall that $\estReward^\window_\arm(\armCount_{\arm,\currentTime})$ is not defined for $h > \armCount_{\arm,\currentTime}$). Finally,  \myAlgorithm selects the active arm $i(t)$ whose number of samples matches the current window, i.e., the least pulled arm in $\mathcal{K}_{h}$. The set of available rewards and the number of pulls are then updated accordingly. 

%

\textbf{Runtime and memory usage} 
At each round~$t$, \myAlgorithm needs to store and update up to $t$ averages per-arm. Since moving from an average computed on window $h$ to $h+1$ can be done incrementally at a cost $\cO(1)$, the worst-case time and memory complexity per round is $\cO(Kt)$, which amounts to a total $\cO(KT^2)$ cost. This is not practical for large $T$.\footnote{This analysis is worst-case. In many cases, the number of samples for the suboptimal arms may be much smaller than $\cO(t)$. For instance, in stochastic bandits it is as little as $\cO(\log t)$, thus reducing the complexity to $\cO(KT\log T)$.} We have a fix.

In App.\,\ref{EFF} we detail \EFF, an efficient variant of \FEWA. \EFF is built around two main ideas.\footnote{As pointed by a reviewer, a similar yet different approach has appeared independently in the context of streaming mining~\citep{bifet2007learning}.} First, at any round $t$ we can avoid calling {\small\textsc{Filter}} for all possible windows $h$ starting from 1 with an increment of 1. In fact, the confidence interval $c(h, \delta_t)$ decreases as $1/\sqrt{h}$ and we could select windows $h$ with an exponential increment so that confidence intervals between two consecutive calls to {\small\textsc{Filter}} have a constant ratio. In practice, we replace the window increment (Line~\ref{algline:window} of \FEWA) by a geometric window $h \triangleq 2^j$. This modification alone is not enough to reduce the computation. While we reduce the number of estimates that we construct, updating $\widehat{\mu}^h_{i}$  from $h=2^j$ to $h=2^{j+1}$ still requires spanning over past samples, thus leading to the same $\cO(Kt)$ complexity in the worst-case. In order to reduce the overall complexity, we avoid recomputing $\widehat{\mu}^h_{i}$ at each call of {\small\textsc{Filter}} and by replacing it with \textit{precomputed} estimates. Whenever $N_{i,t} = 2^j$ for some~$j$, we create an estimate $\hat{s}_{i,j}^{\,c}$ by averaging all the last $N_{i,t}$ samples. These estimates are then used whenever {\small\textsc{Filter}} is called with $h = 2^j$. Instead of updating $\hat{s}_{i,j}^{\,c}$ at each new sample, we create an associated \emph{pending} estimate $\hat{s}_{i,j}^{\,p}$ which averages all the more recent samples. More formally, let $t$ be the round when $N_{i,t}=2^j$, then $\hat{s}_{i,j}^{\,p}$ is initialized at 0 and it then stores the average of all the samples observed from $t$ to $t'$, when $N_{i,t'}=2^{j+1}$ (i.e., $\hat{s}_{i,j}^{\,p}$ is averaging at most $2^j$ samples). At this point, the $2^j$ samples averaged in $\hat{s}_{i,j}^{\,c}$ are \textit{outdated} and they are replaced by the new average $\hat{s}_{i,j}^{\,p}$, which is then reinitialized to 0. The sporadic update of the precomputed estimates and the small number of them drastically reduces per-round time and space complexity to $\cO(K\log t)$. Furthermore, \EFF preservers the same regret guarantees as \FEWA. In the worst case, $\hat{s}_{i,j}^{\,c}$ may not cover the last $2^{j-1}-1$ samples. Nonetheless, the precomputed estimates with smaller windows (i.e., $j'< j$) are updated more frequently, thus effectively covering the $2^{j-1}-1$ samples ``missed'' by $\hat{s}_{i,j}^{\,c}$. As a result, the active sets returned by {\small\textsc{Filter}} are still accurate enough to derive regret guarantees that are only a constant factor worse than \FEWA  (App.\,\ref{EFF}).

\vspace{-0.1in}
\section{Regret Analysis}\label{sec:theory}
\vspace{-0.1in}

We first give problem-independent regret bound for \myAlgorithm and sketch its proof in Sect.\,\ref{sketch}. Then, we derive problem-dependent guarantees in Sect.\,\ref{ss:dep}.

\begin{restatable}{theorem}{restaalgoindepub}
\label{independent_theorem}
For any rotting bandit scenario with means $\{\mu_i(n)\}_{i,n}$ satisfying Asm.\,\ref{assum-Lipschitz} with bounded decay~$L$ and any  horizon $T$, {\myAlgorithm} run with $\alpha= 5$, \textit{i.e.,} $\delta_\currentTime= 1/t^5,$ suffers an expected regret\,\footnote{See Corollary~\ref{cor:HP-minimax} and~\ref{cor:HP-PD} for the high-probability result.} of
\begin{equation*}
\mathbb{E}[\regret(\EWA)] \leq 13\subgaussian(\sqrt{KT} + K)\sqrt{\log(T)}+ 2KL.
\end{equation*}%
\end{restatable}%
\paragraph{Comparison to \citet{levine2017rotting}} The regret of \SWA is bounded by $\tcO(\mu_{\max}^{1/3}K^{1/3} T^{2/3})$ for rotting functions with range in $[0,\mu_{\max}]$. In our setting, we do not restrict rewards to stay positive but we bound the per-round decay by $L$, thus leading to rotting functions with range in $\left[-LT, L\right]$. As a result, when applying \SWA to our setting, we should set $\mu_{\max}=L(T+1)$, which leads to $\cO(T)$ regret, thus showing that according to its original analysis, \SWA may not be able to learn in our general setting. On the other hand, we could use \myAlgorithm in the setting of \citet{levine2017rotting} by setting $L = \mu_{\max}$ as the largest drop that could occur. In this case, \myAlgorithm suffers a regret of $\tcO(\sqrt{KT})$, thus significantly improving over \SWA. The improvement is mostly due to the fact that \myAlgorithm exploits filters using moving averages with increasing windows to discard arms that are suboptimal w.h.p. Since this process is done at each round, \myAlgorithm smoothly tracks changes in the value of each arm, so that if an arm becomes worse later on, other arms would be recovered and pulled again. On the other hand, \SWA relies on a fixed exploratory phase where all arms are pulled in a round-robin way and the tracking is performed using averages constructed with a fixed window. Moreover, \myAlgorithm is anytime, while the fixed exploratory phase of \SWA requires either to know $T$ or to resort to a doubling trick, which often performs poorly in practice. 
\paragraph{Comparison to deterministic rotting bandits}
For $\sigma=0$, our upper bound reduces to $KL$, thus matching the prior (upper and lower) bound of~\citet{heidari2016tight} for deterministic rotting bandits. Moreover, the additive decomposition of regret shows that there is \emph{no coupling} between the stochastic problem and the rotting problem as terms depending on the noise level $\sigma$ are separated from the terms depending on the rotting level $L$, while in \SWA these are coupled by a $L^{1/3}\sigma^{2/3}$ factor in the leading term. 
\paragraph{Comparison to stochastic bandits}
The regret of \FEWA matches the worst-case optimal regret bound of the standard stochastic bandits (i.e., $\mu_i(n)$s are constant) up to a logarithmic factor. Whether an algorithm can achieve $\cO(\sqrt{KT})$ regret bound is an open question. On one hand, \FEWA needs confidence bounds to hold for different windows at the same time, which requires an additional union bound and thus larger confidence intervals w.r.t.\,\UCBone. On the other hand, our worst-case analysis shows that some of the difficult problems that reach the worst-case bound of Thm.\,\ref{independent_theorem} are realized with constant 
functions, which is the standard stochastic bandits, for which \MOSS-like~\citep{audibert2009minimax} algorithms achieve regret guarantees without the $\log T$ factor. Thus, the necessity of the extra $\log T$ factor for the worst-case regret of rotting bandits remains an open problem.

\vspace{-0.1in}
\subsection{Sketch of the proof}
\label{sketch}
\vspace{-0.1in}

We now give a sketch of the proof of our regret bound. We first introduce the expected value of the estimators used in \myAlgorithm. For any $n$ and $1\leq h \leq n$, we define
\begin{equation*}
\expestReward^\window_\arm(n) \triangleq \expectation\left[\estReward^\window_\arm(n) \right] = \frac{1}{\window} \sum_{j=1}^\window \reward_\arm(n -j).
\end{equation*}
Notice that at round $t$, if the number of pulls of arm~$i$ is $N_{i,t}$, then $\expestReward^1_\arm(N_{i,t}) = \mu_i(N_{i,t}-1)$, which is the expected value of arm $i$ the last time it was pulled.
We introduce Hoeffding's concentration inequality and the favorable event that we leverage in the analysis.

\begin{proposition}\label{prop:heoffding}
For any fixed arm $i$, number of pulls~$n$, and window $h$, we have that with probability $1-\delta,$ 
\begin{equation}
\big| \estReward^\window_\arm(n) - \expestReward^\window_\arm(n)\big| \leq c(\window, \delta) \triangleq \sqrt{\frac{2 \subgaussian^2}{\window}\log{\frac{1}{\delta}}}\cdot
\end{equation}
For any round $t$ and confidence $\delta_t \triangleq \delta_0/t^\alpha$, let
\begin{align*}
\!\HPevent_t\! \triangleq\! \Big\{ &\forall i\in\mathcal{K}, \forall n \leq t, \forall h \leq n, \big| \estReward^\window_\arm(n) - \expestReward^\window_\arm(n) \big| \!\leq\! c(\window, \delta_t) \!\Big\}
\end{align*}
be the event under which the estimates constructed by {\FEWA} at round $t$  are all accurate up to $c(h,\delta_t)$. Taking a union bound gives $\mathbb{P}(\HPevent_t) \geq 1- Kt^2\delta_t/2.$ 
\end{proposition}

\paragraph{Active set} We derive an important lemma that provides support for the arm selection process obtained by a series of refinements through the {\small \textsc{Filter}} subroutine. Recall that at any round $t$, after pulling arms $\{ \armCount^{\EWA}_{\arm,\currentTime} \}_i$ the greedy (oracle) policy would select an arm 
\begin{align*}
\arm^\star_\currentTime \pa{\left\{ \armCount^{\EWA}_{\arm,\currentTime} \right\}_i}  \in  \argmax_{\arm \in \possibleArms} \reward_\arm \left( \armCount^{\EWA}_{\arm,\currentTime}\right).
\end{align*}
We denote by $\reward^+_t(\EWA) \triangleq \max_{\arm \in \possibleArms} \reward_\arm ( \armCount^{\EWA}_{\arm,\currentTime}),$ the reward obtained by pulling~$\arm^\star_\currentTime.$ The dependence on $\EWA$ in the definition of $\reward^+_t(\EWA)$ stresses the fact that we consider what the oracle policy would do at the state reached by $\EWA$.
While \myAlgorithm cannot directly match the performance of the oracle arm, the following lemma shows that the reward averaged over the last $h$ pulls of any arm in the active set is close to the performance of the oracle arm up to four times $c(\window,  \delta_\currentTime)$.

\begin{restatable}{lemma}{restafundamentallemma}
\label{fundamental-lemma}
On the favorable event $\HPevent_t$, if an arm~$\arm$ passes through a filter of window $\window$ at round $\currentTime$, i.e., $i\in\ \mathcal{K}_h$, then the average of its $\window$ last pulls satisfies
\begin{equation}\label{eq:fundamental.eq}
\expestReward^{\window}_\arm(\armCount_{\arm,\currentTime}^{\EWA} ) \geq  \reward^+_t(\EWA) - 4 c(\window, \delta_\currentTime).
\end{equation}
\end{restatable}
This result  relies heavily on the non-increasing assumption of rotting bandits. In fact, for any arm $i$ and any window $h$, we have
\begin{equation*}
\wb\mu_i^h(N_{i,t}^{\EWA}) \geq \wb\mu_i^1(N_{i,t}^{\EWA}) \geq \mu_i(N_{i,t}^{\EWA}).
\end{equation*}
While the inequality above for $i_t^*$ trivially satisfies Eq.\,\ref{eq:fundamental.eq}, Lem.\,\ref{fundamental-lemma} is proved by integrating the possible errors introduced by the filter in selecting active arms due to the error of the empirical estimates.

\paragraph{Relating {\myAlgorithm} to the oracle policy}
While Lem.\,\ref{fundamental-lemma} provides a link between the value of the arms returned by the filter and the oracle arm, $i^\star_t$ is defined according to the number of pulls obtained by \myAlgorithm up to $t$, which may significantly differ from the sequence of pulls of the oracle policy. In order to bound the regret, we need to relate the actual performance of the optimal policy to the value of the arms pulled by \myAlgorithm. Let $h_{i,t} \triangleq \abs{ \armCount_{\arm, t}^{\EWA} - \armCount_{\arm, t}^{\star} }$ be the absolute difference in the number of pulls between $\EWA$ and the  optimal policy up to $t$. Since $\sum_{i} \armCount_{\arm, t}^{\EWA} = \sum_{i} \armCount_{\arm, t}^{\star} = t$, we  have that $\sum_{i \in {\overpullSet}}  h_{i,t} = \sum_{i \in {\underpullSet}} h_{i,t}$ which means that there are as many total overpulls as underpulls. Let $j\in\underpullSet$ be an underpulled arm\footnote{If such arm does not exist, then $\EWA$ suffers no regret.} with $N_{j,T}^{\EWA} < N_{j,T}^{\star}$, then, for all $s \in \left\{  0, \ldots, h_{j,t} \right\}$, we have the inequality
%
\begin{align}\label{eq:plus.minus}
 \reward^+_T(\EWA) &= \max_{\arm \in \possibleArms} \reward_\arm(\armCount_{\arm,\timeEnd}^{\EWA}) \geq \reward_j(\armCount_{j,\timeEnd}^{\EWA}+s).
\end{align}
%
As a result, from Eq.\,\ref{eq:regret2} we have the regret upper bound
\begin{align}\label{eq:regret.decomposition}
\regret(\EWA) 
\leq \sum_{\arm\in \overpullSet}   \sum_{h=0}^{\window_{\arm, T} -1} \pa{\reward^+(\EWA) - \reward_\arm(\armCount_{\arm, \timeEnd}^{\policy^\star} + h)},
\end{align}
where we have obtained the inequality  by bounding $\reward_\arm(t') \leq \reward^+_T(\EWA)$ in the first summation 
and then using $\sum_{i \in {\overpullSet}}  h_{i,T} = \sum_{i \in {\underpullSet}} h_{i,T}$. While the previous expression shows that we can just focus on over-pulled arms in $\overpullSet$, it is still difficult to directly control the expected reward $\reward_\arm(\armCount_{\arm, \timeEnd}^{\policy^\star} + h)$, as it may change at each round (by at most~$L$). Nonetheless, we notice that its cumulative sum can be directly linked to the average of the expected reward over a suitable window. In fact, for any $i\in\overpullSet$ and $h_{i,T} \geq 2$, we have
\begin{align*}
(h_{i,T}-1)\expestReward^{h_{i,T}-1}_\arm(\armCount_{\arm,\timeEnd} -1) = \sum_{t'=0}^{h_{i,T}-2} \reward_\arm(\armCount_{\arm,\timeEnd}^{\policy^\star} + t').
\end{align*}
At this point we can control the regret for each $i\in \overpullSet$ in Eq.\,\ref{eq:regret.decomposition} by applying the following corollary of Lem.\,\ref{fundamental-lemma}.

\begin{restatable}{corollary}{restafundamentalcorrelary}\label{fundamental-correlary}
	Let $\arm \in \overpullSet$ be an arm overpulled by {\FEWA} at round $t$ and $\window_{\arm,t} \triangleq \armCount_{\arm, t}^{\EWA} - \armCount_{\arm, t}^{\policy^\star} \geq 1$ be the difference in the number of pulls w.r.t.\,the optimal policy $\pi^\star$ at round $t$. On the favorable event $\HPevent_t$,  we  have
	\begin{align}
	\reward^+_t(\EWA) - \expestReward^{\window_{\arm,t}}_i(\armCount_{\arm,t}) \leq  4 c(\window_{\arm,t}, \delta_t).
	\end{align}
\end{restatable}

\vspace{-0.1in}
\subsection{Problem-dependent guarantees}
\label{ss:dep}
\vspace{-0.1in}

Since our setting generalizes the standard stochastic bandit setting, a natural question is whether we pay any price for this generalization. While the result of~\citet{levine2017rotting} suggested that learning in rotting bandits could be more difficult, in Thm.\,\ref{independent_theorem} we actually proved that \myAlgorithm nearly matches the problem-independent regret $\tcO(\sqrt{\narms\timeEnd})$. We may wonder whether this is true for the \emph{problem-dependent} regret as well.

\begin{remark}\label{remarkUCB}
Consider a stationary stochastic bandit setting with expected rewards $\{\mu_i\}_i$ and $\reward_\star \triangleq \max_\arm \reward_\arm$. Corollary~\ref{fundamental-correlary} guarantees that for $\delta_\currentTime \geq 1/\timeEnd^\alpha,$ 
\begin{align}
\reward_\star - \reward_\arm \leq 4c\pa{\window_{\arm,T}-1,  \delta_\currentTime} = 4\sqrt{\frac{2\alpha\subgaussian^2 \log(\timeEnd)}{\window_{\arm,T} -1}}
\nonumber\\
\text{\ or equivalently,\ }
\label{eq:LaiRob}
\window_{\arm,T} \leq 1+ \frac{32\alpha \subgaussian^2 \log(\timeEnd)}{(\reward_{\star} - \reward_\arm) ^2}\cdot
\end{align}
\end{remark}
Therefore, our algorithm matches the lower bound of~\citet{lai1985asymptotically} up to a constant, thus showing that learning in the rotting bandits are never harder than in the stationary case. Moreover, this upper bound is at most $\alpha$ larger than the one for \UCBone~\citep{auer2002finite}.\footnote{To make the results comparable to the one of~\citet{auer2002finite}, we need to replace $2\subgaussian^2$ by $\nicefrac{1}{2}$ for sub-Gaussian noise.} The main source of suboptimality is the use of a confidence bound filtering instead of an upper-confidence index policy. Selecting the less pulled arm in the active set is conservative as it requires uniform exploration until elimination, resulting in a factor 4 in the confidence bound guarantee on the selected arm (vs.\,2 for \UCB), which implies 4 times more overpulls than \UCB (see Eq.\,\ref{eq:LaiRob}). We conjecture that this may not be necessarily needed and it is an open question whether it is possible to derive either an index policy or a better selection rule. The other source of suboptimality w.r.t.\,\UCB is the use of larger confidence bands because of the higher number of estimators computed at each round ($Kt^2$ instead of $Kt$ for \UCB).


Remark~\ref{remarkUCB} also reveals that Corollary~\ref{fundamental-correlary} can be used to derive a general problem-dependent result in the rotting case.
In particular, with Corollary~\ref{fundamental-correlary} we upper-bound 
the maximum number of overpulls by a problem dependent quantity
\begin{equation}
\label{eq:hit+}
\window_{\arm,T}^+ \triangleq \max \left\{ \window \leq 
1 + \frac{32\alpha \subgaussian^2 \log(T)}{\Delta_{i,h-1}^2} \right\}\CommaBin
\end{equation}
\[
\text{\qquad where \ } \Delta_{i,h} \triangleq \min_{j \in \possibleArms} \reward_j\pa{N_{j,T}^\star -1} - \bar{\reward}_i^h\left( N_{i,t}^\star+h \right).
\]
We then use Corollary~\ref{fundamental-correlary} again to upper-bound the regret caused by $h_{\arm,T}^+$ overpulls for each arm, leading to Theorem~\ref{dependent_theorem} (see the full proof in App.\,\ref{sec:proofdep}). 

\begin{restatable}{theorem}{restaalgoub}\label{dependent_theorem}
For $\alpha = 5$ and $C_\alpha\triangleq 32\alpha\subgaussian^2$, the regret of \FEWA is bounded as
\begin{align*}
\mathbb{E}\left[R_T(\EWA)\right]  \leq \sum_{\arm\in \possibleArms} \pa{\frac{C_{5}\log(T)}{\Delta_{i,h_{i,T}^+-1}}+ \sqrt{C_{5}\log(T)} + 2L}.
\end{align*}

\end{restatable}

\newpage
\vspace{-0.1in}
\section{Numerical simulations}
\label{Simulation}
\vspace{-0.1in}


\begin{figure*}[t]
\vspace{-0.15in}
\includegraphics[trim={0.3cm 0 1.5cm 0},clip,width=0.32\textwidth]{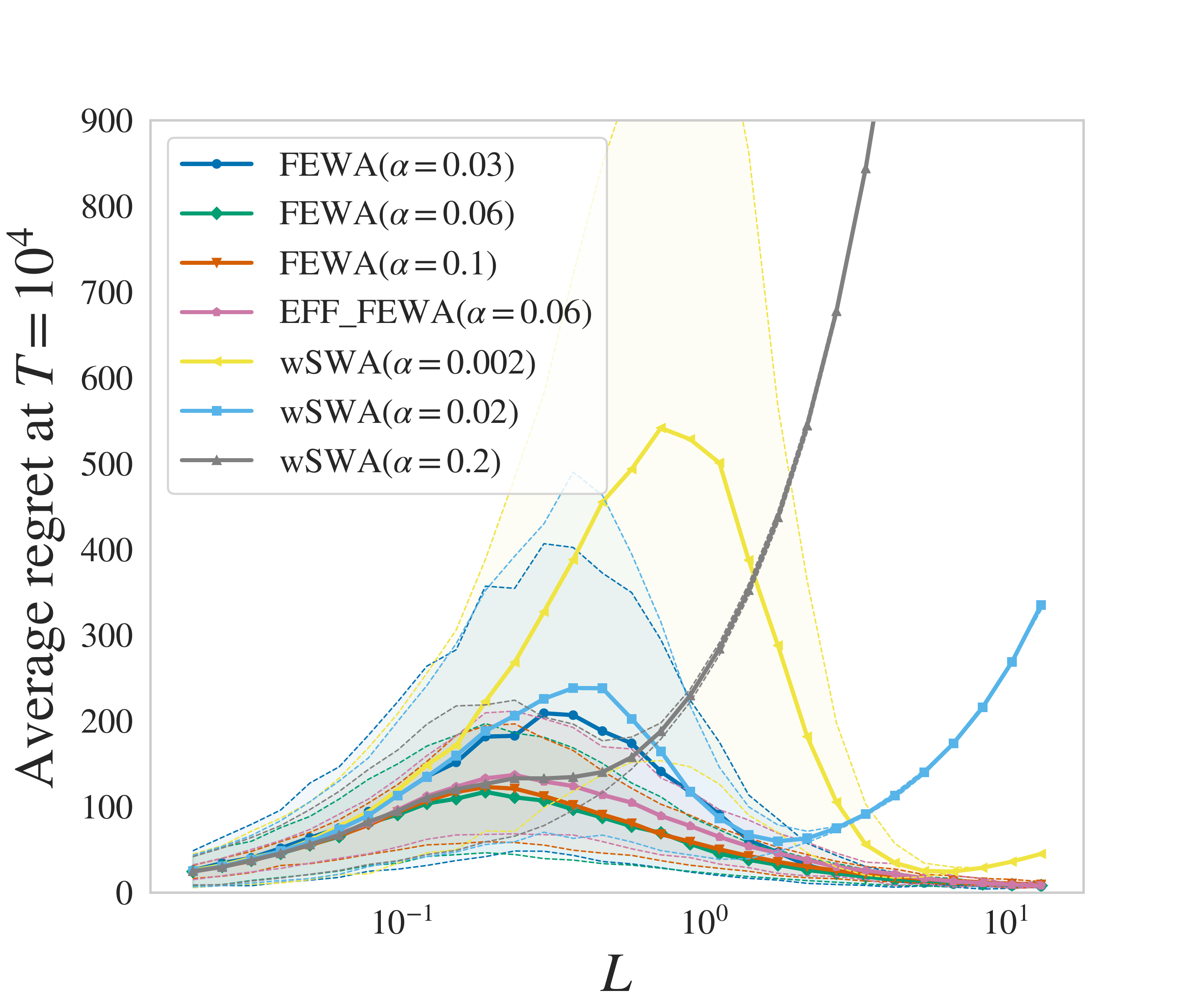}
\includegraphics[trim={0.5cm 0 1.5cm 0},clip,width=0.32\textwidth]{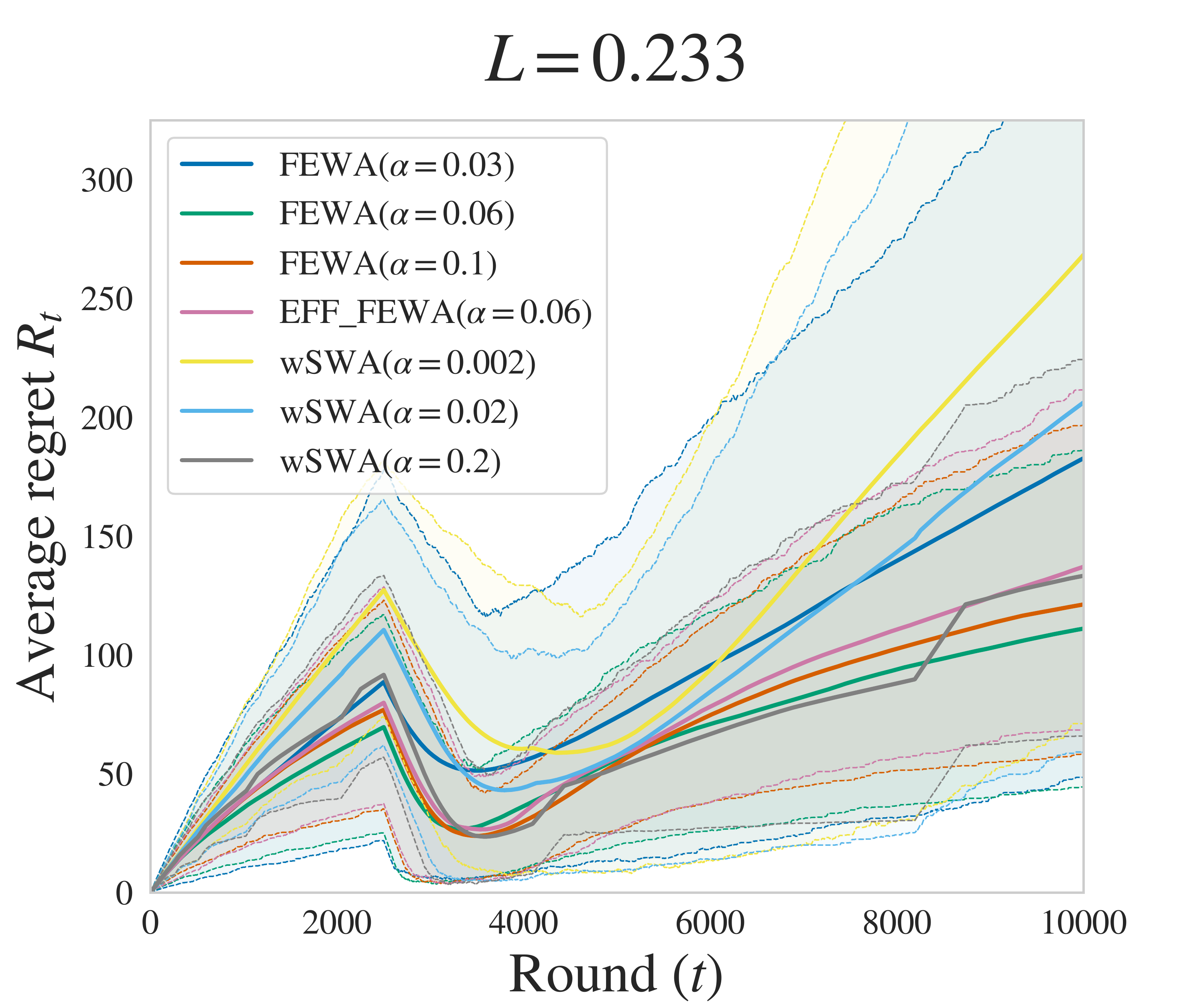}
\includegraphics[trim={0.5cm 0 0.3cm 0},clip,width=0.33\textwidth]{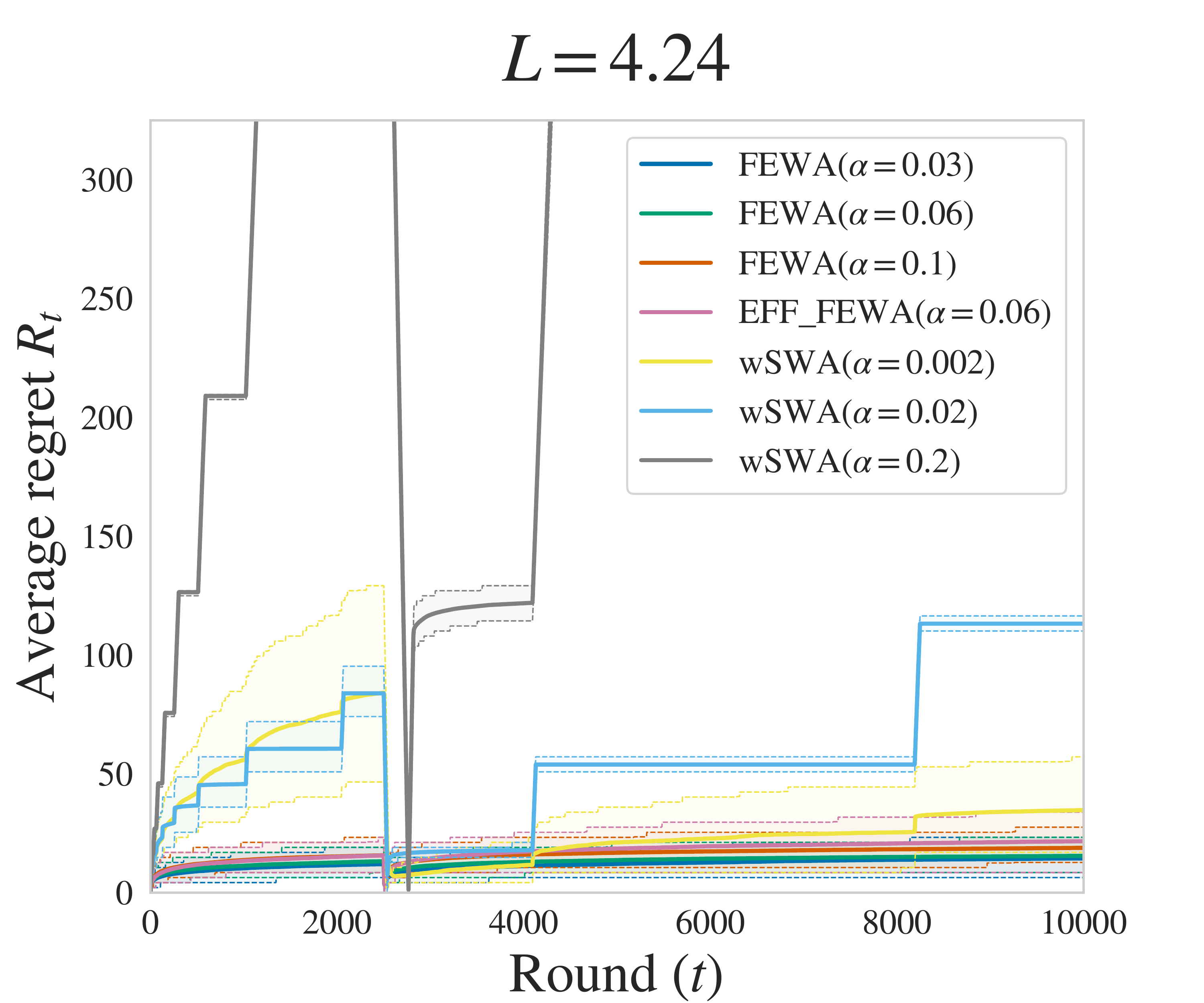}
\caption{Comparison between \FEWA and \SWA in the 2-arm setting. \textbf{Left:} Regret at $T=10000$ for different values of $L$.
 \textbf{Middle-right:} Regret over time for $L = 0.23$ (worst case for \FEWA) and $L = 4.24$ (case of $L \gg \sigma$). We highlight the $\left[10\%, 90\%\right]$ confidence region.}
\label{rotting-fig}
\end{figure*}

\paragraph{\underline{2-arms}} We design numerical simulations to study the difference between \SWA and \myAlgorithm. We consider rotting bandits  with two arms defined as
\begin{align*}
\reward_1(n) = 0, \;\; \forall n \leq \timeEnd \quad \text{and}  \quad
\reward_2(n) = \begin{cases}
\frac{\lipschitz}{2}  & \text{ if } n < \frac{\timeEnd}{4}\CommaBin \\
-\frac{\lipschitz}{2} & \text{ if } n \geq \frac{\timeEnd}{4}\cdot
\end{cases}
\end{align*}
The rewards are then generated by applying a Gaussian i.i.d.\,noise $\mathcal{N}\left(0,\subgaussian = 1\right)$. The single point of non-stationarity in the second arm is designed to satisfy Asm.\,\ref{assum-Lipschitz} with a bounded decay $L$. It is important to notice that in this specific case, $L$ also plays the role of defining the gap $\Delta$ between the arms, which is known to heavily impact the performance both in  stochastic bandits and in the rotting bandits  (Cor.\,\ref{dependent_theorem}). In particular, for any learning strategy, the gap between the two arms is always $\Delta = |\mu_1(n_1) - \mu_2(n_2)| = L/2$. Recall  that in stochastic bandits, the problem independent bound $\cO(\sqrt{KT})$ is obtained by the worst-case choice of $\Delta \triangleq \sqrt{K/T}$.
In the two-arm setting defined above, the optimal allocation is $N_{1,T}^\star = 3T/4$ and $N_{2,T}^\star=T/4$. 

\paragraph{Algorithms} Both algorithms have a parameter $\alpha$ to tune. In \SWA, $\alpha$ is a multiplicative constant to tune the window. We test three different values of $\alpha$, including the recommendation of~\citet{levine2017rotting}, $\alpha = 0.2$. In general, the smaller the $\alpha$, the smaller the averaging window and the more reactive the algorithm is to large drops. Nonetheless, in stationary regimes, this may correspond to high variance and poor regret. On the other hand, a large value of $\alpha$ may reduce variance but increase the bias in case of rapidly rotting arms. Thm.\,3.1 of~\citet{levine2017rotting} reveals this trade-off in the regret bound of \SWA, which has a factor $(\alpha \mu_{\max} + \alpha^{-1/2})$, where $\mu_{\max}$ is the largest arm value. The best choice of $\alpha$ is then $\mu_{\max}^{-2/3}$, which reduces the previous constant to $\mu_{\max}^{1/3}$. In our experiment, $\mu_{\max} = L$ and we could expected that for any fixed~$\alpha$, \SWA may perform well in cases when $\alpha \approx \mu_{\max}^{-2/3}$, while the performance may degrade for larger $\mu_{\max}$.

In \myAlgorithm, $\alpha$ tunes the confidence $\delta_\currentTime = 1/(t^\alpha)$ used in $c(h, \delta_t)$. While our analysis suggests $\alpha = 5$, the analysis of confidence intervals, union bounds, and filtering algorithms is too conservative. Therefore, we use more aggressive values, $\alpha \in  \left\{ 0.03, 0.06, 0.1\right\}.$

\paragraph{Experiments} In Fig.\,\ref{rotting-fig}, we compare the performance of the two algorithms and their dependence on $L$. The first plot shows the regret at $T$ for various values of~$L$. The second and the third plot show the regret as a function of $T$ for $L = 0.23$ and $L=4.24$, which corresponds to the  worst empirical performance for \myAlgorithm and to the $L \gg \sigma$ regime respectively. All experiments have  $T=10 000$ and are averaged over $500$ runs.

Before discussing the results, we point out that in the rotting setting, the regret can increase and decrease over time. Consider two simple policies: $\pi_1$, which first pulls arm $1$ for $N^\star_{1,T}$ times and then arm $2$ for $N^\star_{2,T}$ times, and $\pi_2$ in reversed order (first arm $2$ and then arm $1$). If we take $\pi_2$ as reference, $\pi_1$ has an increasing regret for the first $T/4$ rounds, which then would plateau from $T/4$ up to $3T/4$ as both $\pi_1$ and $\pi_2$ are pulling arm $1$. Then from $3T/4$ to $T$, the regret of~$\pi_1$ would reverse back to~0 since $\pi_2$ would keep selecting arm $1$ and getting a reward of $0$, while $\pi_1$ transitions to pulling arm $2$ with a reward of~$L/2$.

\begin{figure*}[t]
\vspace{-0.15in}
	\centering
	\includegraphics[clip, width = \textwidth]{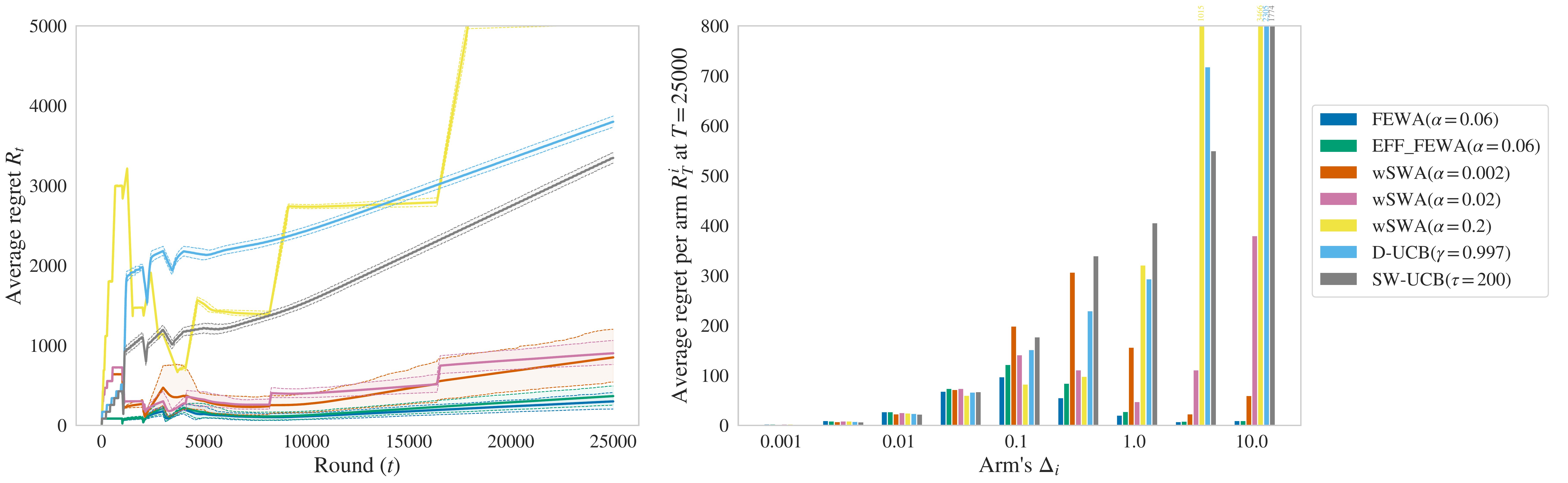}	
	\caption{10-arm setting. \textbf{Left:} Regret over time. \textbf{Right:} Regret per arm at the end of the experiment.}
	\label{benchmark}
	\vspace{-0.15in}
\end{figure*}


\paragraph{Results} Fig.\,\ref{rotting-fig} shows that the performance of \SWA depends on the proper tuning of $\alpha$ w.r.t.\,$\mu_{\max} = L$, as predicted by Thm.\,3.1 of~\citet{levine2017rotting}. In fact, for small values of $L$, the best choice is $\alpha=0.2$, while for larger values of $L$ a smaller $\alpha$ is preferable. In particular, when $L$ grows very large, the regret tends to grow linearly with $L$. On the other hand, \myAlgorithm seems much more robust to different values of~$L$. Whenever~$T$ and $\sigma$ are large compared to $L$, Thm.\,\ref{independent_theorem} suggests that the regret of \FEWA is dominated by $\cO(\sigma\sqrt{KT})$, while the term $KL$ becomes more relevant for large values of the drop $L$. We also notice that since $L$ defines the gap between the value of $\mu_1$ and $\mu_2$, the problem-independent bound is achieved for the worst-case choice of $L\sim 2\sqrt{K/T}$, when the regret of \FEWA is indeed the largest. Fig.\,\ref{rotting-fig} middle and right confirm these findings for the extreme choice of the worst-case value of $L$ and the regime where the drop is much larger than the noise level, i.e., where the term $KL$ dominates the regret.

We conclude that \FEWA is more robust than \SWA as it almost always achieves the best performance across different problems while being agnostic to the value of $L$. On the other hand, \SWA's performance is very sensitive to the choice of~$\alpha$ and the same value of the parameter may correspond to significantly different performance depending on $L$.
Finally, we notice that \EFF has a regret comparable to \FEWA. 

\paragraph{\underline{10-arms}} We also tested a rotting setting with 10 arms. The mean of 1 arm is constant with value 0 while the means of 9 arms abruptly decrease after 1000 pulls from $+\Delta_i$ to $-\Delta_i$. $\Delta_i$  is ranging from 0.001 to 10 in a geometric sequence. In this setting, the regret can be written as $R_T(\pi) = \sum_{i=1}^9 h_{i,T}\Delta_i$. Hence, the regret per arm  
is $R_T^i(\pi) \triangleq \Delta_i h_{i,T}$.
In Fig.\,\ref{benchmark}, we compare the performance of different algorithms for their best parameter. The left plot shows the average regret as a function of time. The right plot shows the regret per arm (indexed by $\Delta_i$) at the end of the experiment.
\vfil

\paragraph{Results} On Fig.\,\ref{benchmark} (left), we see that \FEWA outperforms \SWA. On the right, we remark that no tuning of \SWA is able to perfom well for all $\Delta_i$s. It is comparable with the 2-arms experiments where no tuning is good on all the experiments.  We also test \SWUCB and \DUCB \citep{garivier2011upper-confidence-bound} with parameters tuned for this experiment. While the two algorithms are known benchmarks for non-stationary \emph{restless} bandits, they are penalized in our \emph{rested} bandits problem. Indeed, they keep exploring arms that have not been pulled for many rounds which is detrimental in our case as the arms stay constant when they are not pulled. Hence, there is no good choice for their forgetting parameters: A fast forgetting rate makes the policies repeatedly pull bad arms (whose mean rewards do not change when they are not pulled in the rested setting) while a slow forgetting rate makes the policies not able to adapt to abrupt shifts.

Last, we remark that \EFF is penalized for arms with small $\Delta_i$, for which the impact of the delay is more significant. At the end of the game, \EFF suffers $22\%$ more regret but reduce the computational time by $99.7\%$ (Table~\ref{time-table}).
\renewcommand{\tabcolsep}{4pt}
\begin{center}
\begin{table}
\begin{tabular}{|c|c|c|c|c|}
  \hline
  \!\FEWA\! & \!\EFF\! & \!\SWA\!  & \!\SWUCB\! & \!\DUCB\! \\
  \hline
2271&7&1&5&2\\
  \hline
\end{tabular}
\vspace{0.2em}
  \caption{Average running time for the 10-arms experiment in seconds.}
  \label{time-table}
\end{table}
\end{center}

\vspace{-0.1in}
\section{Conclusion}
\label{conclusions}
\vspace{-0.1in}

We introduced \myAlgorithm, a novel algorithm for the non-parametric rotting bandits. 
We proved that \myAlgorithm achieves an $\tcO(\sqrt{KT})$ regret without any knowledge of the decays by using moving averages with a window that effectively adapts to the changes in the expected rewards. This result greatly improves over the \SWA algorithm by~\citet{levine2017rotting}, that suffers a regret of order $\tcO(K^{1/3}T^{2/3})$. Thus our result shows that the rotting bandit scenario is not harder than the stochastic one. Our technical analysis of \myAlgorithm hinges on the \textit{adaptive} nature of the window size. The most interesting aspect of the proof technique is that confidence bounds are used not only for the action selection but also for the \textit{data} selection, i.e., to identify the best window to trade off the bias and the variance in estimating the current value of each arm. 

\vfil
\newpage
\paragraph{Acknowledgements} 
We thank Nir Levine for his helpful remarks and Lilian Besson for his bandits package  \citep{SMPyBandits}. The research presented was supported by European CHIST-ERA project DELTA, French Ministry of Higher Education and Research, Nord-Pas-de-Calais Regional Council, Inria and Otto-von-Guericke-Universit\"at Magdeburg associated-team north-European project Allocate, and French National Research Agency projects ExTra-Learn (n.ANR-14-CE24-0010-01) and BoB (n.ANR-16-CE23-0003). The work of A.\,Carpentier is also partially supported by the Deutsche Forschungsgemeinschaft (DFG) Emmy Noether grant MuSyAD (CA 1488/1-1), by the DFG - 314838170, GRK 2297 MathCoRe, by the DFG GRK 2433 DAEDALUS, by the DFG CRC 1294 Data Assimilation, Project A03, and by the UFA-DFH through the French-German Doktorandenkolleg CDFA 01-18.
This research has also benefited from the support of the FMJH Program PGMO and from the support to this program from Criteo.
Part of the computational experiments were conducted using the Grid'5000 experimental
testbed (\url{https://www.grid5000.fr}).
\bibliography{library}
\bibliographystyle{plainnat}
\newpage
\onecolumn
\appendix

\onecolumn
\section{Proof of core {\myAlgorithm} guarantees}
\restafundamentallemma*
\begin{proof}
Let $\arm$ be an arm that passed a filter of window $\window$ at round $\currentTime$.
First, we use the confidence bound for the estimates and we pay the cost of keeping all the arms up to a distance $2c(\window,  \delta_\currentTime)$ of $\estReward^{\window}_{\max,\currentTime}$,
\begin{equation}
\label{1}
\expestReward^{\window}_\arm(\armCount_{\arm,\currentTime} )\geq \estReward^{\window}_\arm(\armCount_{\arm,\currentTime} )- c(\window,  \delta_\currentTime) 
\geq \estReward^{\window}_{\max,\currentTime} - 3c(\window, \delta_\currentTime)
\geq \max_{\arm \in \possibleArms_\window}\expestReward^{\window}_{\arm}(\armCount_{\arm,\currentTime})  - 4 c(\window, \delta_\currentTime),
\end{equation}
where in the last inequality, we used that for all $i \in \possibleArms_\window,$ \[\estReward^{\window}_{\max,\currentTime} \geq  \estReward^{\window}_{\arm}(\armCount_{\arm,\currentTime}) \geq \expestReward^{\window}_{\arm}(\armCount_{\arm,\currentTime})  - c(\window, \delta_\currentTime).\]
Second, since the means of arms are decaying, we know that 
\begin{equation}
\label{2}
 \reward^+_t(\EWA) \triangleq \reward_{\arm^\star_\currentTime}(\armCount_{\arm^\star_\currentTime,\currentTime}) \leq  \reward_{\arm^\star_\currentTime}(\armCount_{\arm^\star_\currentTime,\currentTime} - 1) =  \expestReward^{1}_{\arm^\star_\currentTime}(\armCount_{\arm^\star_\currentTime,\currentTime})  \leq \max_{\arm \in \possibleArms}\expestReward^{1}_{\arm}(\armCount_{\arm,\currentTime}) = \max_{\arm \in \possibleArms_1}\expestReward^{1}_{\arm}(\armCount_{\arm,\currentTime}).
\end{equation}
Third, we show that the largest average of the last $\window'$ means of arms in $\possibleArms_{\window'}$ is increasing with~$\window'$,
\begin{equation*}
\forall  \window' \leq \armCount_{\arm,\currentTime}-1,  \max_{\arm \in \possibleArms_{\window'+1}}\expestReward^{\window'+1}_{\arm}(\armCount_{\arm,\currentTime})   \geq \max_{\arm \in \possibleArms_{\window'}}\expestReward^{\window'}_{\arm}(\armCount_{\arm,\currentTime}). 
\end{equation*}
To show the above property, we remark that thanks to our selection rule, the arm that has the largest average of means, always passes the filter. Formally, we show that $\argmax_{\arm \in \possibleArms_{\window'}}\expestReward^{\window'}_{\arm}(\armCount_{\arm,\currentTime}) \subseteq \possibleArms_{\window'+1}.$ 
Let $\arm^{\window'}_{\max} \in \argmax_{\arm \in \possibleArms_{\window'}}\expestReward^{\window'}_{\arm}(\armCount_{\arm,\currentTime})$. Then for such $\arm^{\window'}_{\max}$, we have
\begin{equation*}
\estReward^{\window'}_{\arm^{\window'}_{\rm max}}(\armCount_{\arm^{\window'}_{\max},\currentTime} ) \geq \expestReward^{\window'}_{\arm^{\window'}_{\max}}(\armCount_{\arm^{\window'}_{\max},\currentTime}) - c(\window', \delta_\currentTime) 
\geq \expestReward^{\window'}_{\max,\currentTime} - c(\window',  \delta_\currentTime) \geq \estReward^{\window'}_{\max,\currentTime}- 2c(\window',  \delta_\currentTime),
\end{equation*}
where the first and the third inequality are due to confidence bounds on estimates, while the second one is due to the definition of $i^{\window'}_{\rm max}$. 

Since the arms are decaying, the average of the last $\window' +1$ mean values for a given arm is always greater than the average of the last $\window'$ mean values
and therefore, 
\begin{equation}
\label{3}
 \max_{\arm \in \possibleArms_{\window'}}\expestReward^{\window'}_{\arm}(\armCount_{\arm,\currentTime}) =   \expestReward^{\window'}_{\arm^{\window'}_{\rm max}}(\armCount_{\arm^{\window'}_{\rm max},\currentTime} ) \leq \expestReward^{\window'+1}_{\arm^{\window'}_{\rm max}}(\armCount_{\arm^{\window'}_{\rm max},\currentTime} ) \leq \max_{\arm \in \possibleArms_{\window'+1}}\expestReward^{\window' +1 }_{\arm}(\armCount_{\arm,\currentTime}), 
\end{equation}
because $\arm^{\window'}_{\rm max} \in \possibleArms_{\window'+1}$. Gathering Equations~\ref{1}, \ref{2}, and~\ref{3} leads to the claim of the lemma,
\begin{equation*}
\expestReward^{\window}_\arm(\armCount_{\arm,\currentTime} )
\stackrel{\eqref{1}}{\geq} \max_{\arm \in \possibleArms_\window}\expestReward^{\window}_{\arm}(\armCount_{\arm,\currentTime})  - 4c(\window, \delta_\currentTime)
\stackrel{\eqref{3}}{\geq} \max_{\arm \in \possibleArms_1}\expestReward^{1}_{\arm}(\armCount_{\arm,\currentTime}) - 4c(\window,  \delta_\currentTime)
\stackrel{\eqref{2}}{\geq}  \reward^+_t(\EWA) - 4c(\window, \delta_\currentTime).
\end{equation*}\end{proof}

\restafundamentalcorrelary*

\begin{proof}
If $i$ was pulled at round $t$, then by the condition at Line~\ref{algline:condition} of Algorithm~\ref{EWA}, it means that $i$ passes through all the filters from $h=1$ up to $N_{i,t}$. In particular, since $1 \leq h_{i,t} \leq N_{i,t}$, $i$ passed the filter for $h_{i,t}$, and thus we can apply Lemma~\ref{fundamental-lemma} and conclude 
\begin{equation}
\expestReward^{\window}_\arm(\armCount_{\arm,\currentTime} ) \geq  \reward^+_t(\EWA) - 4 c(\window_{i,t}, \delta_\currentTime).
\end{equation}
\end{proof}

\section{Proofs of auxiliary results}
\begin{lemma}
\label{reg:Decompo}
Let $h_{i,t}^\pi \triangleq | N_{i,T}^\pi - N_{i,T}^{\star}|$. For any policy $\pi$, the regret at round T is no bigger than
\begin{equation*}
R_T(\policy) \leq \sum_{i \in \overpullSet} \sum_{h=0}^{h_{i,T}^\policy-1}\left[\HPevent_{t_i^\policy(N_{i,T}^\star + h)} \right]\left(\mu^+_T(\policy) - \mu_i(N_{i,T}^{\star} + h ) \right) + \sum_{t=1}^T \Big[\bar{\HPevent_t}\Big]Lt.
\end{equation*}
We refer to the the first sum above as to $A_\policy$ and to the second one as to $B$.
\end{lemma}
\begin{proof}
We consider the regret at round $T$. From Equation~\ref{eq:regret2}, the decomposition of regret in terms of overpulls and underpulls
gives
\begin{equation*}
\regret(\policy) = \sum_{\arm\in \underpullSet} \sum_{t'=\armCount_{\arm, \timeEnd}^{\policy}}^{\armCount_{\arm, \timeEnd}^{\star}-1} \reward_\arm(t') - \sum_{\arm\in \overpullSet} \sum_{t'=\armCount_{\arm, \timeEnd}^{\star}}^{\armCount_{\arm, \timeEnd}^{\policy}-1} \reward_\arm(t'). 
\end{equation*}
In order to separate the analysis for each arm, we upper-bound all the rewards in the first sum by their maximum $\reward^+_T(\policy) \triangleq \max_{i\in\possibleArms} \reward_i(N_{i,T}^\policy)$. This upper bound is tight for problem-independent bound because one cannot hope that the unexplored reward would decay to reduce its regret in the worst case. 
We also notice that there are as many terms in the first double sum (number of underpulls) than in the second one (number of overpulls). 
This number is equal to  $\sum_{\overpullSet} \window_{\arm,T}^\policy$.
Notice that this does \emph{not} mean that for each arm $\arm$, the number of overpulls equals to the number of underpulls, which cannot happen anyway since an arm cannot be simultaneously underpulled and overpulled. Therefore, we keep only the second double sum,
\begin{equation}
\label{eq:RegretDecompo}
\regret(\policy) \leq \sum_{\arm\in \overpullSet}   \sum_{t'=0}^{\window_{\arm,T}^\policy-1} \pa{\reward^+_T(\pi) - \reward_\arm(\armCount_{\arm, \timeEnd}^{\star} + t')}.
\end{equation}
Then, we need to separate overpulls that are done under $\HPevent_t$ and under $\bar{\HPevent_t}$. We introduce $t_i^{\policy}(n)$, the round at which $\policy$ pulls arm $i$ for the $n$-th time. We now make the round at which each overpull occurs  explicit,
\begin{align*}
\regret(\policy) & \leq \sum_{\arm\in \overpullSet}   \sum_{t'=0}^{\window_{\arm,T}^\policy-1} \sum_{t=1}^T \left[ t_i^{\policy}\pa{\armCount_{\arm, \timeEnd}^{\star} + t'} = t \right]  \pa{\reward^+_T(\policy) - \reward_\arm(\armCount_{\arm, \timeEnd}^{\star} + t')}\\
& \leq \underbrace{\sum_{\arm\in \overpullSet}   \sum_{t'=0}^{\window_{\arm,T}^\policy-1} \sum_{t=1}^T \left[ t_i^{\policy}\pa{\armCount_{\arm, \timeEnd}^{\star} + t'} = t \land \HPevent _t \right] \pa{\reward^+_T(\policy) - \reward_\arm(\armCount_{\arm, \timeEnd}^{\star} + t')}}_{A_\policy}\\
&+ \underbrace{\sum_{\arm\in \overpullSet}\sum_{t'=0}^{\window_{\arm,T}^\policy-1} \sum_{t=1}^T \left[ t_i^{\policy}\pa{\armCount_{\arm, \timeEnd}^{\star} + t'} = t \land \bar{\HPevent _t} \right]\pa{\reward^+_T(\policy) - \reward_\arm(\armCount_{\arm, \timeEnd}^{\star} + t')}}_B.
\end{align*}
For the analysis of the pulls done under $\HPevent_t$ we do not need to know at which round it was done. Therefore, 
\[
A_\policy \leq \sum_{\arm\in \overpullSet}   \sum_{t'=0}^{\window_{\arm,T}^\policy-1}  \left[ \HPevent_{t(N_{i,t}^\star + t')} \right] \pa{\reward^+_T(\policy) - \reward_\arm(\armCount_{\arm, \timeEnd}^{\star} + t')}.
\]
For \myAlgorithm, it is not easy to directly guarantee the low probability of overpulls (the second sum). Thus, we upper-bound the regret of each overpull at round $t$ under $\bar{\HPevent_t}$ by its maximum value $Lt$. While this is done to ease \myAlgorithm analysis, this is valid for any policy $\policy$. Then, noticing that we can have at most 1 overpull per round $t$, i.e., $\sum_{\arm\in \overpullSet}\sum_{t'=0}^{\window_{\arm,T}^\policy-1}\left[ t_i^{\policy}\pa{\armCount_{\arm, \timeEnd}^{\star} + t'} = t  \right] \leq 1$, we get
\[
B \leq  \sum_{t=1}^T \Big[\bar{\HPevent_t}\Big] Lt\pa{\sum_{\arm\in \overpullSet}\sum_{t'=0}^{\window_{\arm,T}^\policy-1}\left[ t_i^{\policy}\pa{\armCount_{\arm, \timeEnd}^{\star} + t'} = t  \right]} \leq  \sum_{t=1}^T \Big[\bar{\HPevent_t}\Big] Lt.
\]
Therefore, we conclude that
\[
\regret(\policy) \leq \underbrace{\sum_{\arm\in \overpullSet}   \sum_{t'=0}^{\window_{\arm,T}^\policy-1}  \left[ \HPevent_{t_i^\policy(N_{i,t}^\star + t')} \right] \pa{\reward^+_T(\policy) - \reward_\arm(\armCount_{\arm, \timeEnd}^{\star} + t')}}_{A_{\policy}} + \underbrace{\sum_{t=1}^T \Big[\bar{\HPevent_t}\Big] Lt}_B.
\]
\end{proof}
\begin{lemma}
\label{lemma:A}
Let $h_{i,t} \triangleq h_{i,t}^{\EWA} = | N_{i,T}^{\EWA} - N_{i,T}^{\star}|$. For policy $\EWA$ with parameters ($\alpha$, $\delta_0$), $A_{\EWA}$ defined in  Lemma~\ref{reg:Decompo} is upper-bounded by
\begin{align*}
A_{\EWA} &\triangleq  \sum_{\arm\in \overpullSet}   \sum_{t'=0}^{\window_{\arm,T}-1}  \left[ \HPevent_{t_i^{\EWA}(N_{i,t}^\star + t') }\right] \pa{\reward^+_T(\EWA) - \reward_\arm(\armCount_{\arm, \timeEnd}^{\star} + t')} \\
& \leq \sum_{\arm\in \overpullSet_\HPevent} \pa{4\sqrt{2\alpha\subgaussian^2\log_+(\timeEnd\delta_0^{-1/\alpha})}
+ 4\sqrt{2\alpha\subgaussian^2\left(h_{i,T}^\HPevent -1\right)\log_+(\timeEnd\delta_0^{-1/\alpha})} + \lipschitz}.
\end{align*}
\end{lemma}
\begin{proof}
First, we define $h_{i,T}^\HPevent \triangleq \max\left\{ h \leq h_{i,T} | \ \HPevent_{t_i^{\EWA}(N_{i,t}^\star + h)}\right\}$, the last overpull of arm $i$ pulled at round $t_i \triangleq t_i^{\EWA}(N_{i,t}^\star + h_{i,T}^\HPevent) \leq T$ under $\HPevent_t$. Now, we upper-bound $A_{\EWA}$ by including all the overpulls of arm $i$ until the $h_{i,T}^\HPevent$-th overpull, even the ones under $\bar{\HPevent_t}$,
\begin{align*}
A_{\EWA} &\triangleq  \sum_{\arm\in \overpullSet}   \sum_{t'=0}^{\window_{\arm,T}^{\EWA}-1}  \left[ \HPevent_{t_i^{\EWA}(N_{i,t}^\star + t') }\right] \pa{\reward^+_T(\EWA) - \reward_\arm(\armCount_{\arm, \timeEnd}^{\star} + t')} 
\leq \sum_{\arm\in \overpullSet_\HPevent}   \sum_{t'=0}^{h_{i,T}^\HPevent-1}  \pa{\reward^+_T(\EWA) - \reward_\arm(\armCount_{\arm, \timeEnd}^{\star} + t')},
\end{align*}
where $\overpullSet_\HPevent \triangleq \left\{ i \in \overpullSet | \  h_{i,T}^\HPevent \geq 1 \right\}.$ We can therefore split the second sum of~$h_{i,T}^\HPevent$ term above  into two parts. The first part corresponds to the first $h_{i,T}^\HPevent-1$ (possibly zero) terms (overpulling differences) and the second part to the last  $(h_{i,T}^\HPevent-1)$-th one. Recalling that at round $t_i$, arm $i$ was selected under $\HPevent_{t_i}$, we apply
Corollary~\ref{fundamental-correlary} to bound the regret caused by previous overpulls of $i$ (possibly none),
\begin{align}
A_{\EWA} &\leq  \sum_{i \in \overpullSet_\HPevent}   \reward^+_T(\EWA) - \reward_i\pa{N_{i, T}^\star + h_{i,T}^\HPevent  -1} + 4\pa{h_{i,T}^\HPevent - 1}c\pa{h_{i,T}^\HPevent-1, \delta_{t_i }} \label{eq:cor1-use1}\\
&\leq \sum_{i \in \overpullSet_\HPevent}   \reward^+_T(\EWA) - \reward_i\pa{N_{i, T}^\star + h_{i,T}^\HPevent  -1} + 4\pa{h_{i,T}^\HPevent - 1}c\pa{h_{i,T}^\HPevent-1, \delta_{T}}\\
&\leq \sum_{i \in \overpullSet_\HPevent}   \reward^+_T(\EWA) - \reward_i\pa{N_{i, T}^\star + h_{i,T}^\HPevent  -1} + 4\sqrt{2\alpha\subgaussian^2\pa{h_{i,T}^\HPevent - 1}\log_+{\pa{T\delta_0^{-1/\alpha}}}}\CommaBin
\label{eq:lasttimepdq}
\end{align}
with $\log_+(x) \triangleq \max(\log(x),0)$. The second inequality is obtained because $\delta_t$ is decreasing and $c(.,.,\delta)$ is decreasing as well. The last inequality is the definition of confidence interval in Proposition~\ref{prop:heoffding} with $\log_+(T^\alpha)\leq \alpha \log_+(T)$ for $\alpha>1$.
 If  $\armCount_{\arm, \timeEnd}^{\star} = 0$ and $h_{i,T}^\HPevent = 1$ then
\[ \reward^+_T(\EWA) - \reward_\arm(\armCount_{\arm, \timeEnd}^{\star} + h_{i,T}^\HPevent - 1) =  \reward^+(\EWA) - \reward_\arm(0) \leq \lipschitz,\] 
since and  $\reward^+(\EWA) \leq \lipschitz$ and  $ \reward_\arm(0) \geq 0$ by the assumptions of our setting.
Otherwise, we can decompose 
\begin{align*}
\reward^+_T(\EWA) - &\reward_\arm(\armCount_{\arm, \timeEnd}^{\star} + h_{i,T}^\HPevent - 1) 
= \underbrace{\reward^+_T(\EWA) - \reward_\arm(\armCount_{\arm, \timeEnd}^{\star} + h_{i,T}^\HPevent-2)}_{A_1} + 
\underbrace{\reward_\arm(\armCount_{\arm, \timeEnd}^{\star} + h_{i,T}^\HPevent-2) -  \reward_\arm(\armCount_{\arm, \timeEnd}^{\star} + h_{i,T}^\HPevent - 1)}_{A_2}.
\end{align*}
For term $A_1$, since arm $\arm$ was overpulled at least once by \myAlgorithm, it passed at least the first filter. Since this $h_{i,T}^\HPevent$-th overpull is done under $\HPevent_{t_i}$, by Lemma~\ref{fundamental-lemma}  we have that
\[
A_1 \leq 4c(1, \delta_{t_i}) \leq 4c(1,\timeEnd^{-\alpha}) \leq 4\sqrt{2\alpha\subgaussian^2\log_+\left(\timeEnd\delta_0^{-1/\alpha}\right)} .
\] 
The second difference, 
$A_2 = \reward_\arm(\armCount_{\arm, \timeEnd}^{\star} + h_{i,T}^\HPevent-2) -  \reward_\arm(\armCount_{\arm, \timeEnd}^{\star} + h_{i,T}^\HPevent - 1 )$  
cannot exceed $\lipschitz$, since by the assumptions of our setting, the maximum decay in one round is bounded.
Therefore, we further upper-bound Equation~\ref{eq:lasttimepdq} as
\begin{align}
A_{\EWA} \leq \sum_{\arm\in \overpullSet_\HPevent} \pa{4\sqrt{2\alpha\subgaussian^2\log_+{\pa{T\delta_0^{-1/\alpha}}}}
+ 4\sqrt{2\alpha\subgaussian^2\left(h_{i,T}^\HPevent -1\right)\log_+\pa{T\delta_0^{-1/\alpha}}} + \lipschitz}.
\label{HPevent0}
\end{align}
\end{proof}

\begin{lemma}
\label{lemma:B}
Let $\zeta(x) = \sum_n n^{-x}$. Thus, with $\delta_t = \delta_0/(Kt^\alpha)$ and $\alpha > 4$, we can use Proposition~\ref{prop:heoffding} and get
\[
\EE B \triangleq \sum_{t=1}^T p\pa{\bar{\HPevent_t}}Lt \leq \sum_{t=1}^T \frac{Lt\delta_0}{2t^{\alpha-2}}\leq L \delta_0 \frac{\zeta(\alpha -3)}{2}\cdot
\]
\end{lemma}

\section{Minimax regret analysis of \myAlgorithm}
\label{proof1}
\restaalgoindepub*
\label{proof2} 

\begin{proof}

To get the problem-independent upper bound for \myAlgorithm, 
we need to upper-bound the regret by quantities which do not depend on $\left\{\reward_\arm\right\}_i$. 
The proof is based on Lemma~\ref{reg:Decompo},
where we bound the expected values of terms 
$A_{\EWA}$ and $B$ from the statement of the lemma.
We start by noting that on high-probability event $\HPevent_\timeEnd$, we have by Lemma~\ref{lemma:A} and $\alpha = 5$ that

\begin{equation*}
A_{\EWA}  \leq \sum_{\arm\in \overpullSet_\HPevent} \pa{  4\sqrt{10\subgaussian^2\log(T)}   
+ 4\sqrt{10\subgaussian^2\left(\window_\arm -1\right)\log(T)} +  \lipschitz}.
\end{equation*}
Since $\overpullSet_\HPevent \subseteq \overpullSet$ and there are at most $\narms - 1$ overpulled arms, we can upper-bound the number of terms in the above sum by  $\narms - 1$.
Next, the total number of overpulls $\sum_{\arm\in\overpullSet} \window_{\arm,T}$ cannot exceed $\timeEnd$. 
As square-root function is concave we can use Jensen's inequality. 
Moreover, we can deduce that the worst allocation of overpulls is the uniform one, i.e., $\window_{\arm,T} = \timeEnd/(\narms-1),$
\begin{align}
A_{\EWA} &\leq (\narms -1)(4\sqrt{10\subgaussian^2\log(T)} + \lipschitz) + 4\sqrt{10\subgaussian^2\log(T)} \sum_{\arm\in \overpullSet} \sqrt{(\window_{\arm,T} - 1)}\nonumber\\ 
&\leq (\narms -1)(4\sqrt{10\subgaussian^2\log(T)} + \lipschitz) + 4\sqrt{10\subgaussian^2\left(\narms-1\right)\timeEnd\log(T)}.
\label{Abound}
\end{align}
Now, we consider the expectation of term $B$  from Lemma~\ref{reg:Decompo}. According to Lemma~\ref{lemma:B}, with $\alpha=5$ and $\delta_0 =1$,
\begin{equation}
\label{Bbound}
\EE B\leq \frac{L\zeta(2)}{2} = \frac{L\pi^2}{12} \cdot 
\end{equation}
Therefore, using Lemma~\ref{reg:Decompo} together with Equations~\ref{Abound} and \ref{Bbound}, we bound the total expected regret as
\begin{equation}
\mathbb{E}[\regret(\EWA)] \leq 4\sqrt{10\subgaussian^2\left(\narms-1\right)\timeEnd\log(T)} + (\narms -1)(4\sqrt{10\subgaussian^2\log(T)} + \lipschitz) +\frac{\lipschitz\pi^2}{6}\cdot
\end{equation}
\end{proof}

\begin{corollary}
\label{cor:HP-minimax}
\myAlgorithm run with $\alpha > 3$ and $\delta_0 \triangleq 2\delta/\zeta(\alpha-2)$ achieves with probability $1 - \delta$, 
\[
R_T(\EWA) = A_{\EWA} \leq 4\sqrt{2\alpha\subgaussian^2\log_+\pa{\frac{\timeEnd}{\delta_0^{1/\alpha}}}}\pa{K -1 +\sqrt{(K-1)T}} + \pa{K-1} \lipschitz.
\]
\end{corollary}
\begin{proof}
We consider the event $\bigcup_{t\leq T} \HPevent_t$  which happens with probability
\begin{equation*}
 1 - \sum_{t\leq T} \frac{Kt^2\delta_t}{2} \leq 1 - \sum_{t\leq T} \frac{Kt^2\delta_t}{2} \leq 1 - \frac{\zeta(\alpha - 2)\delta_0}{2}\cdot
\end{equation*}
Therefore, by setting $\delta_0 \triangleq 2\delta/\zeta(\alpha-2)$, we have that $B = 0$ with probability $1-\delta$ since $\Big[ \bar{\HPevent_t} \Big] = 0$ for all $t$. We can then use the same analysis of $A_{\EWA}$ as in Theorem~\ref{independent_theorem} to get
\[
R_T(\EWA) = A_{\EWA} \leq 4\sqrt{2\alpha\subgaussian^2\log_+\pa{\frac{\timeEnd}{\delta_0^{1/\alpha}}}}\pa{K -1 +\sqrt{(K-1)T}} + \pa{K-1} \lipschitz.
\]

\end{proof}

\section{Problem-dependent regret analysis of \myAlgorithm} 
\label{sec:proofdep}
\begin{restatable}{lemma}{lemmaPDbound}\label{lemma:PDbound}
$A_{\EWA}$ defined in Lemma~\ref{reg:Decompo} is upper-bounded by a problem-dependent quantity,
\[
A_{\EWA} \leq \sum_{\arm\in \possibleArms} \pa{\frac{32\alpha\subgaussian^2\log_+(T\delta_0^{-1/\alpha})}{\Delta_{i,h_{i,T}^+-1}} + \sqrt{32\alpha\subgaussian^2\log_+(T\delta_0^{-1/\alpha})}}+  (K-1)\lipschitz.
\]
\end{restatable}
\begin{proof}

We start from the result of Lemma~\ref{lemma:A},
\begin{equation}
\label{lemmaAstart}
A_{\EWA} \leq \sum_{\arm\in \overpullSet_\HPevent} \pa{4\sqrt{2\alpha\subgaussian^2\log_+(\timeEnd\delta_0^{-1/\alpha})} \left( 1 + \sqrt{h_{i,T}^\HPevent -1}\right)} + (K-1)\lipschitz. 
\end{equation}
We want to bound $h_{i,T}^\HPevent $ with a problem dependent quantity $h_{i,T}^+$. We remind the reader that for arm $i$ at round $T$, the $h_{i,T}^\HPevent$-th overpull has been on $\HPevent_{t_i}$ pulled at round $t_i$. Therefore, Corollary~\ref{fundamental-correlary} applies and we have
\begin{align*}
\expestReward_\arm^{h_{i,T}^\HPevent  - 1} \left( \armCount_{\arm,\timeEnd}^{\star} + h_{i,T}^\HPevent   - 1 \right) &\geq \reward_T^+(\EWA) - 4c\pa{h_{i,T}^\HPevent   - 1, \delta_{t_i}}
\geq \reward_T^+(\EWA) - 4c\pa{h_{i,T}^\HPevent   - 1, \delta_T}
\\&
\geq \reward_T^+(\EWA) - 4\sqrt{\frac{2\alpha\subgaussian^2\log_+\pa{\timeEnd\delta_0^{-1/\alpha}}}{h_{i,T}^\HPevent -1}}
\geq \reward^-_T(\pi^\star) - 4\sqrt{\frac{2\alpha\subgaussian^2\log_+\pa{\timeEnd\delta_0^{-1/\alpha}}}{h_{i,T}^\HPevent -1}} \CommaBin
\end{align*}
with $\reward^-_T(\policy^\star) \triangleq \min_{i \in \possibleArms} \reward_i \pa{N_{i,T}^\star-1}$ being the lowest mean reward for which a noisy value was ever obtained by the optimal policy. $\reward^-_T(\policy^\star) < \reward^+_T(\EWA)$ implies that the regret is 0. Indeed, in that case the next possible pull 
with the largest mean for $\EWA$ is \emph{strictly larger}  than the mean of 
the last pull for $\policy^\star.$ Thus, there is no underpull at this round for $\EWA$ and $R_T(\EWA) = 0$ according to Equation~\ref{eq:regret2}. Therefore, we can assume $\reward^-_T(\policy^\star) \geq \reward^+_T(\EWA)$ for the regret bound.  Next, we define $\Delta_{i,h} \triangleq \reward^-_T(\policy^\star) - \bar{\reward}_i^h\left( N_{i,t}^\star+h \right)$ as the difference between the lowest mean value of the arm pulled by $\policy^\star$ and the average of the $h$ first overpulls of arm~$i$. Thus, we have the following  bound  for $h_{i,T}^\HPevent,$ 
\[
h_{i,T}^\HPevent \leq 1 + \frac{32\alpha\subgaussian^2\log_+\pa{T\delta_0^{-1/\alpha}}}{\Delta_{i,h_{i,T}^\HPevent-1}}\cdot
\] 
Next, $h_{i,T}^\HPevent$ has to be smaller than the maximum such $h$, for which the  inequality just above is satisfied if we replace $h_{i,T}^\HPevent$ by $h$. Therefore,
\begin{equation}
\label{PDoverpulls}
wh_{i,T}^\HPevent \leq h_{i,T}^+  \triangleq \max \left\{ h \leq T \big| \ h \leq  1 + \frac{32\alpha \subgaussian^2 \log_+ \pa{T\delta_0^{-1/\alpha}}}{\Delta_{i,h-1}^2} \right\}\cdot
\end{equation}
Since the square-root function is increasing, we can upper-bound Equation~\ref{eq:lasttimepdq} by replacing $h_{i,T}^\HPevent$ by its upper bound $h_{i,T}^+$
to get
\begin{align*}
A_{\EWA} &\leq \sum_{\arm\in \overpullSet_\HPevent} \pa{4\sqrt{2\alpha\subgaussian^2\log_+(T\delta_0^{-1/\alpha})} \left( 1 + \sqrt{h_{i,T}^+ - 1}\right) + \lipschitz}\\
& \leq \sum_{\arm\in \overpullSet_\HPevent} \pa{\sqrt{32\alpha\subgaussian^2\log_+(T\delta_0^{-1/\alpha})} \left( 1 + \frac{\sqrt{32\alpha\subgaussian^2\log_+(T\delta_0^{-1/\alpha})}}{\Delta_{i,h_{i,T}^+-1}}\right) + \lipschitz}. 
\end{align*}
The quantity $\overpullSet_\HPevent$ is depends on the execution. Notice that there  are at most $K-1$ arms in $\overpullSet_\HPevent$ and that $\overpullSet \subset \possibleArms$. Therefore, we have 
\begin{equation*}
A_{\EWA} \leq \sum_{\arm\in \possibleArms} \pa{\frac{32\alpha\subgaussian^2\log_+\pa{T\delta_0^{-1/\alpha}}}{\Delta_{i,h_{i,T}^+-1}} + \sqrt{32\alpha\subgaussian^2\log_+\pa{T\delta_0^{-1/\alpha}}}}+ (K-1)\lipschitz. 
\end{equation*}
\end{proof}
\restaalgoub*
\begin{proof}
Using Lemmas~\ref{reg:Decompo}, \ref{lemma:B}, and~\ref{lemma:PDbound} we get
\begin{align*}
\EE{\regret(\EWA)} &=\EE{A_{\EWA}} + \EE B 
\leq \sum_{\arm\in \possibleArms} \pa{\frac{32\alpha\subgaussian^2\log(T)}{\Delta_{i,h_{i,T}^+-1}} + \sqrt{32\alpha\subgaussian^2\log(T)}}+ (K-1)\lipschitz  + \frac{L\pi^2}{6} \\
&\leq \sum_{\arm\in \possibleArms} \pa{\frac{32\alpha\subgaussian^2\log(T)}{\Delta_{i,h_{i,T}^+-1}} + \sqrt{32\alpha\subgaussian^2\log(T)} + L}\cdot
\end{align*}
\end{proof}

\begin{corollary}
\label{cor:HP-PD}
\myAlgorithm run with $\alpha > 3$ and $\delta_0 \triangleq 2\delta/\zeta(\alpha-2)$ achieves with probability $1 - \delta$, 
\[
R_T(\EWA) \leq \sum_{\arm\in \possibleArms} \pa{\frac{32\alpha\subgaussian^2\log_+\left(\frac{T\zeta(\alpha-2)^{1/\alpha}}{(2\delta)^{1/\alpha}}\right)}{\Delta_{i,h_{i,T}^+-1} } + \sqrt{32\alpha\subgaussian^2\log_+\left(\frac{T\zeta(\alpha-2)^{1/\alpha}}{(2\delta)^{1/\alpha}}\right) }}+ (K-1)\lipschitz. 
\]
\end{corollary}
\begin{proof}
We consider the event $\cup_{t\leq T} \HPevent_t$  which happens with probability
\begin{equation*}
 1 - \sum_{t\leq T} \frac{Kt^2\delta_t}{2} \leq 1 - \sum_{t\leq T} \frac{Kt^2\delta_t}{2} \leq 1 - \frac{\zeta(\alpha - 2)\delta_0}{2}\cdot
\end{equation*}
Therefore, by setting $\delta_0 \triangleq 2\delta/\zeta(\alpha-2)$, we have that with probability $1-\delta$, $B = 0$ since $\Big[ \bar{\HPevent_t} \Big] = 0$ for all $t$. We use Lemma~\ref{lemma:PDbound} to get the claim of the corollary.
\end{proof}

\section{Efficient algorithm \EFF}
\label{EFF}

\begin{figure}[t]
	\centering
	\includegraphics[scale=0.6]{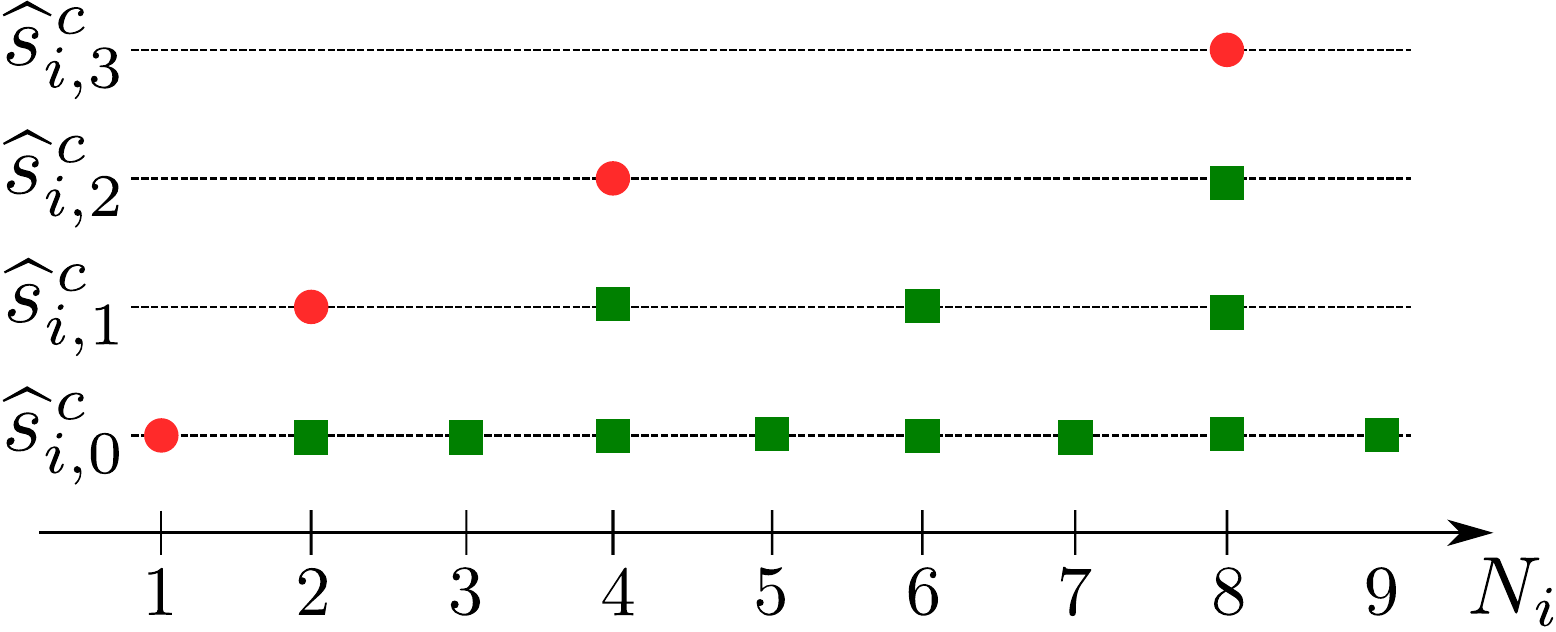}
	\caption{Illustration of the functioning of \EFF. The red circles denotes the number of pulls of arm $i$ at which a new estimate $\widehat{s}^{\,c}_{i,j}$ is created corresponding to a window $h = 2^j$, while the green boxes indicate the number of pulls for which $\hat{s}_{i,j}^{c}$ is updated with the last $2^j$ samples.}
	\label{eff-fig}
\end{figure}

In Algorithm~\ref{Eff-alg}, we present \EFF, an algorithm that stores at most $2 K\log_2(t)$ statistics. More precisely, for $j \leq \log_2(N_{i,t}^{\EFFP})$,  we let $\hat{s}_{i,j}^{\,p}$ and $\hat{s}_{i,j}^{\,c}$ be the current and pending $j$-th statistic for arm $i$. 

\begin{algorithm}[H]
\caption{\EFF}
\label{Eff-alg}
\begin{algorithmic}[1]
\REQUIRE $\possibleArms$, $\delta_0$, $\alpha$
	\STATE pull each arm once, collect reward, and initialize $N_{\arm,K} \leftarrow 1$ 
	\FOR{$\currentTime \gets K+1, K+2, \dots \do $}
		\STATE $\delta_t \leftarrow \delta_0/(t^\alpha)$
		\STATE $j \leftarrow 0$ 
		{\footnotesize \COMMENT{\emph{initialize bandwidth}}}
		\STATE $\possibleArms_1 \leftarrow \possibleArms$ 
		{\footnotesize \COMMENT{\emph{initialize with all the arms}}}
		\STATE $\arm(t) \gets {\tt none}$
		\WHILE{$\arm(t)$ is  ${\tt none}$}
			\STATE $\possibleArms_{2^{j+1}} \leftarrow {\small{\textsc{EFF\_Filter}}}(\possibleArms_{2^j} ,j, \delta_\currentTime)$
			\STATE $j \leftarrow j + 1$ 
			\IF{$\exists \arm \in \possibleArms_{2^j}$ such that $\armCount_{\arm, t}\leq 2^j $}\label{eff:cond}
			\STATE $\arm(t) \leftarrow \arm$
			\ENDIF
		\ENDWHILE
		\STATE  receive $\obs_\arm(\armCount_{\arm,\currentTime +1 }) \leftarrow \obs_{\arm(\currentTime),\currentTime}$
		\STATE ${\small{\textsc{EFF\_Update}}}\left(i(t),r_i(N_{i,t+1}),t+1\right)$
	\ENDFOR
\end{algorithmic}

\end{algorithm}
\noindent
As $N_{i,t}$ increases, new statistics $\hat{s}_{i,j}^{\,c}$ for larger windows are created as illustrated in Figure~\ref{eff-fig}. First, at any time $t$, $\hat{s}_{i,j}^{\,c}$ is the average of $2^{j-1}$ consecutive reward samples for arm $i$ within the last $2^{j} -1$ sample. These statistics are used in the filtering process as they are representative of exactly $2^{j-1}$ recent samples.  Second, $\hat{s}_{i,j}^{\,p}$ stores the pending samples that are not yet taken into account by $\hat{s}_{i,j}^{\,c}$. Therefore, each time we pull arm $i$, we update all the pending averages. When the pending statistic is the average of the $2^{j-1}$ last samples then we set $\hat{s}_{i,j}^{\,c} \leftarrow \hat{s}_{i,j}^{\,p}$ and reinitialize $\hat{s}_{i,j}^{\,p} \leftarrow 0$. 

\begin{algorithm}[b!]
\caption{{\sc EFF\_Filter}}
\label{alg:eff-filter}
\begin{algorithmic}[1]
\REQUIRE $\possibleArms_{2^j}$, $j$, $\delta_t$, $\subgaussian$
\STATE $c(2^j, \delta_\currentTime) \leftarrow \sqrt{2\subgaussian^2/2^j \log{\delta^{-1}_\currentTime}}$
\STATE $\hat{s}_{\max,j}^{\,c}   \leftarrow \max_{\arm \in \possibleArms_{\window}}\hat{s}_{i,j}^{\,c} $
\FOR{$ \arm \in \possibleArms_{\window}$}
	\STATE $\Delta_\arm \leftarrow  \hat{s}_{\max,j}^{\,c}   - \hat{s}_{i,j}^{\,c}$
	\IF{$\Delta_\arm \leq 2c(2^j,  \delta_\currentTime) $}
	\STATE add $\arm$ to $\possibleArms_{2^{j+1}}$
	\ENDIF
\ENDFOR
\ENSURE $\possibleArms_{2^{j+1}}$
\end{algorithmic}
\end{algorithm}
\begin{algorithm}[H]
\caption{{\small\sc EFF\_Update}}
\label{alg:eff-update}
\begin{algorithmic}[1]
\REQUIRE $i$, $r$, $t$
\STATE $\armCount_{\arm(\currentTime),\currentTime} \leftarrow \armCount_{\arm(\currentTime),\currentTime-1} +1$
\STATE $R^{\rm total}_i \leftarrow R^{\rm total}_i + r$ 
{\footnotesize \COMMENT{\emph{keep track of total reward}}}
\IF{$\exists j$ such that $N_{i,t} = 2^j$}
\STATE $\hat{s}_{i,j}^{\,c} \leftarrow R_i^{\rm total}/N_{i,t}$ 
{\footnotesize \COMMENT{\emph{initialize new statistics}}}
\STATE $\hat{s}_{i,j}^{\,p} \leftarrow 0$
\STATE $n_{i,j}\leftarrow 0$
\ENDIF
\FOR{$j \gets 0 \dots \log_2(N_{i,t})$}
\STATE $n_{i,j} \leftarrow n_i +1$
\STATE $\hat{s}_{i,j}^{\,p} \leftarrow \hat{s}_{i,j}^{\,p} + r$
\IF{$n_{i,j} = 2^j$}
\STATE $\hat{s}_{i,j}^{\,c} \leftarrow \hat{s}_{i,j}^{\,p}/2^j$
\STATE $n_{i,j} \leftarrow 0$
\STATE $\hat{s}_{i,j}^{\,p} \leftarrow 0$
\ENDIF
\ENDFOR
\end{algorithmic}
\end{algorithm}

In analyzing the performance of \EFF, we have to account for two different effects: (1) the loss in resolution due to windows of size that increases exponentially instead of a fixed increment of 1, and (2) the delay in updating the statistics $\hat{s}_{i,j}^{\,c}$, which do not include the most recent samples. We let $\expestReward_i^{h',h''}$ be the average of the samples between the $h'$-th to last one and the $h''$-th to last one (included) with $h'' > h'$. \FEWA  was controlling  $\expestReward_i^{1,h}$ for each arm, \EFF controls $\expestReward_i^{h'_i,h'_i +2^{j-1}}$ with different $h'_i \leq 2^{j-1} -1$ for each arm depending on when $\hat{s}_{i,j}^{\,c}$ was refreshed last time. However, since the means of arms are non-increasing, we can consider the worst case when the arm
with the highest mean available at that round is estimated with its last samples (the smaller ones) and the bad arms are estimated on their oldest possibles samples (the larger ones).

\begin{restatable}{lemma}{restalemefficient}
\label{lem:efficient}
On favorable event $\HPevent_t$, if an arm $\arm$ passes through a filter of window $\window$ at round $\currentTime$, the average of its $\window$ last pulls cannot deviate significantly from the best available arm $\arm^\star_\currentTime$ at that round,
\begin{equation*}
\expestReward_i^{2^{j-1},2^{j}-1} \geq \reward^+_t(\EWA) - 4 c(\window, \delta_\currentTime).
\end{equation*}
\end{restatable}
We proceed with modifying Corollary~\ref{fundamental-correlary} to have the following efficient version.
\begin{restatable}{corollary}{restaeffcor}
\label{eff-corollary}
Let $\arm \in \overpullSet$ be an arm overpulled by {\EFF} at round $t$ and $\window_{\arm,t}^{\EFFP} \triangleq \armCount_{\arm, t}^{\EFFP} - \armCount_{\arm, t}^{\star} \geq 1$ be the difference in the number of pulls w.r.t.\,the optimal policy $\pi^\star$ at round $t$. On  favorable event $\HPevent_t$,  we have that 
\begin{align*}
\reward^+_t(\EFFP) - \expestReward^{\window_{\arm,t}^{\EFFP}}(\armCount_{\arm,t}) \leq  \frac{4\sqrt{2}}{\sqrt{2}-1} c(\window_{\arm,t}^{\EFFP}, \delta_t).
\end{align*}
\end{restatable}

\begin{proof}
If $i$ was pulled at round $t$, then by the condition at Line~\ref{eff:cond} of Algorithm~\ref{Eff-alg}, it means that $i$ passes through all the filters until at least window $2^f$ such that $2 ^f \leq h_{i,t}^{\EFFP}< 2^{f+1} $. Note that for $h_{i,t}^{\EFFP} = 1$, then \EFF
has the same guarantee as \FEWA since the first filter is always up to date. Then for $h_{i,t}^{\EFFP} \geq 2,$  
%
\begin{align}
\expestReward^{ 1, h_{i,t}^{\EFFP}}_i(\armCount_{\arm,\currentTime} ) &\geq \expestReward^{ 1, 2^f -1}_i(\armCount_{\arm,\currentTime} )  = \frac{\sum_{j=1}^{f} 2^{j-1} \expestReward_i^{2^{j-1},2^{j}-1}}{2^{f}-1} \label{eq:line1}\\
&\geq \reward^+_t(\EFFP) -\frac{4\sum_{j=1}^{f} 2^{j-1}c(2^{j-1}, \delta)}{2^{f}-1} = \reward^+_t(\EFFP) -4c(1,\delta_t) \frac{\sum_{j=1}^{f} \sqrt{2}^{j-1}}{2^{f}-1} \label{eq:line2}\\
& = \reward^+_t(\EFFP) -4c(1,\delta_t) \frac{\sqrt{2}^{f} -1}{(2^{f}-1)(\sqrt{2}-1)} \geq \reward^+_t(\EFFP) -4c(1,\delta_t) \frac{1}{\sqrt{2}^{f}(\sqrt{2}-1)} \label{eq:line3}\\
& =\reward^+_t(\EFFP) -\frac{4\sqrt{2}}{\sqrt{2}-1}c\pa{2^{f+1},\delta_t} \geq \reward^+_t(\EFFP) -\frac{4\sqrt{2}}{\sqrt{2}-1}c\pa{h_{i,t}^{\EFFP},\delta_t},\label{eq:line4}
\end{align}
where Equation~\ref{eq:line1} uses that the average of older means is larger than average of the more recent ones and then decomposes $2^f-1$ means onto a geometric grid. Then, Equation~\ref{eq:line2} uses Lemma~\ref{lem:efficient} and makes  the dependence of $c(2^{j-1}, \delta)$ on $j$ explicit. Next, Equations~\ref{eq:line3}
and~\ref{eq:line4}  use standard algebra to get a lower bound
and that $c(h,\delta)$ decreases with $h$. 
\end{proof}

\noindent Using the result above, we follow the same proof as the one for \FEWA and derive minimax and problem-dependent upper bounds for \EFF using Corollary~\ref{eff-corollary} instead of Corollary~\ref{fundamental-correlary}.

\begin{restatable}[minimax guarantee for {\EFF}]{corollary}{restacoreffminimax}
\label{cor:eff-minimax}
For any rotting bandit scenario with means $\{\mu_i(n)\}_{i,n}$ satisfying Assumption~\ref{assum-Lipschitz} with bounded decay $L$ and any time horizon $T$, {\EFF} with $\delta_\currentTime= 1/(t^5)$, $\alpha= 5$, and $\delta_0=1$ has its expected regret upper-bounded as
\begin{equation*}
\mathbb{E}\left[\regret\left(\EFFP\right)\right] \leq 13\subgaussian\pa{\frac{\sqrt{2}}{\sqrt{2}-1}\sqrt{KT} + K}\sqrt{\log(T)}+ KL.
\end{equation*}
\end{restatable}

\begin{restatable}[problem-dependent guarantee for {\EFF}]{corollary}{restacoreffpd}
\label{cor:eff-PD}
For $\delta_\currentTime = 1/(\currentTime^5)$, the regret of {\EFF} is upper-bounded as
\begin{align*}
R_T(\EFFP)  \leq \sum_{\arm\in \possibleArms} \pa{\frac{C_{5}\frac{2}{3-2\sqrt{2}}\log(T)}{\Delta_{i,h_{i,T}^+-1}}+ \sqrt{C_{5}\log(T)} + L}\CommaBin
\end{align*}
with $C_\alpha\triangleq 32\alpha\subgaussian^2$ and $h_{i,T}^+$ defined in Equation~\ref{eq:hit+}.
\end{restatable}

\section{Numerical simulations: Stochastic bandits}


\begin{figure}[H]
	\centering
	\includegraphics[scale=0.25]{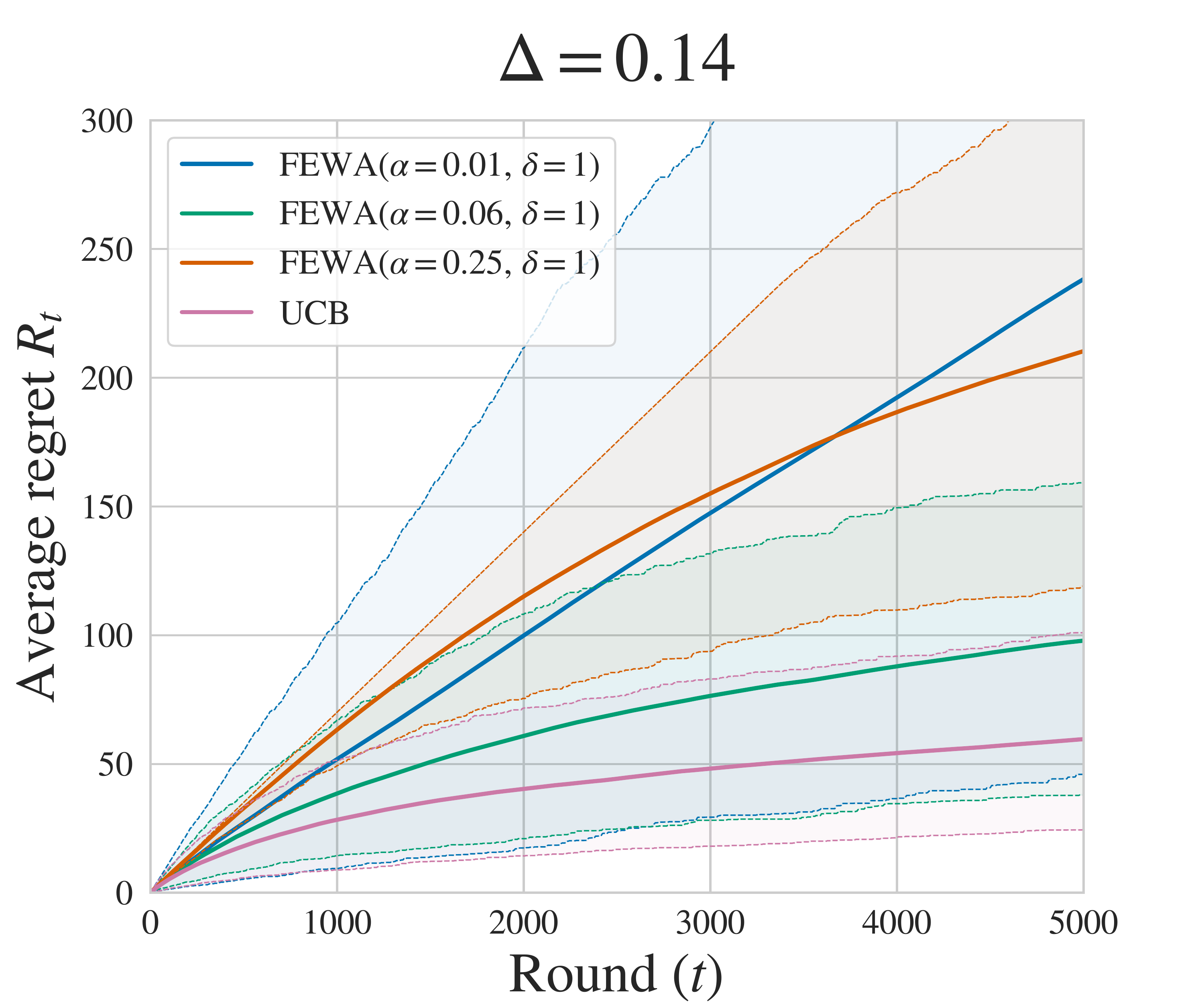}
	\includegraphics[scale=0.25]{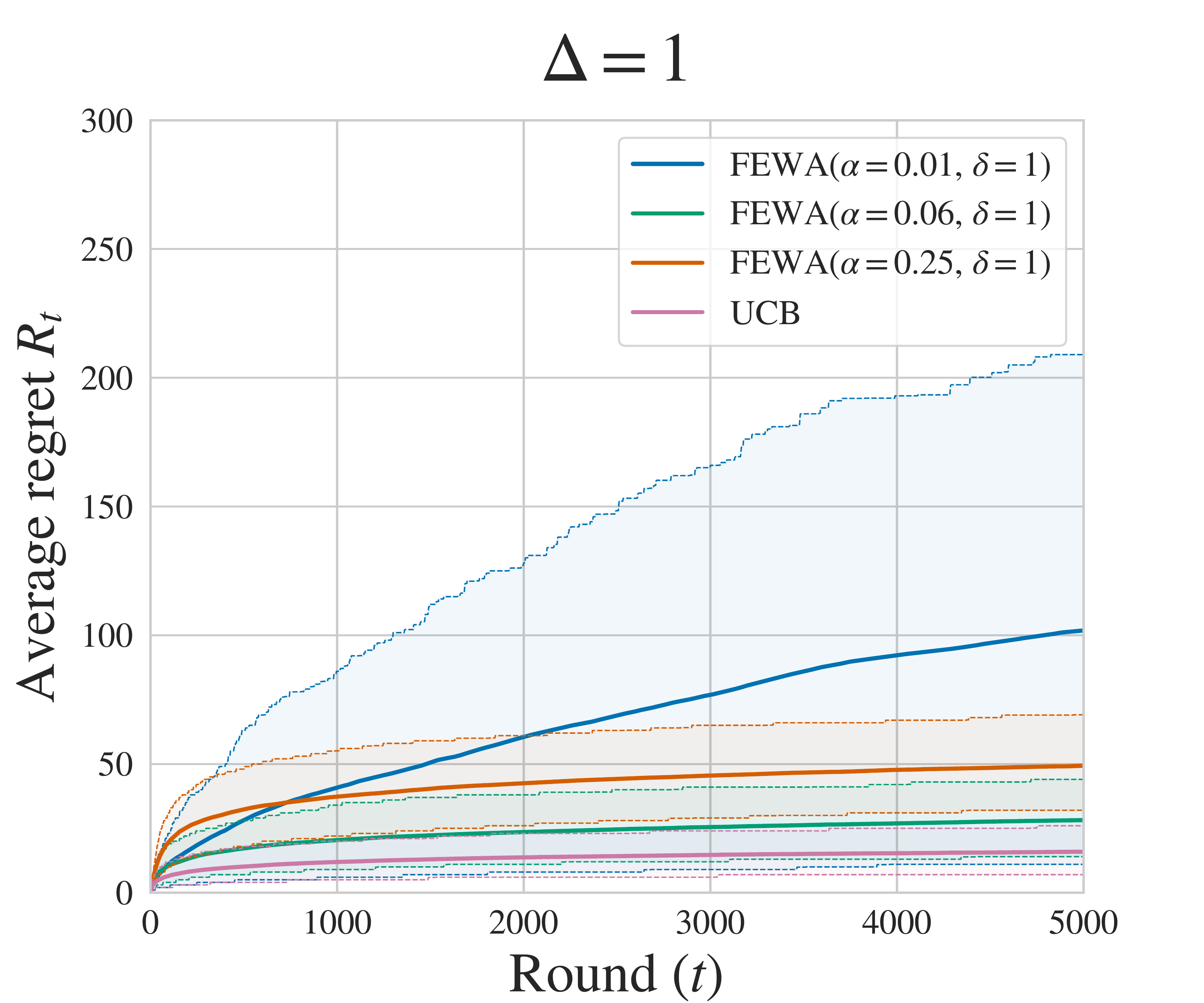}
	\caption{Comparing \UCBone and \myAlgorithm with $\Delta=0.14$ and $\Delta=1$.}
	\label{stochastic-fig}
\end{figure}

In Figure~\ref{stochastic-fig} we compare the performance of \FEWA against \UCB~\citep{lai1985asymptotically} on two-arm bandits with different gaps. These experiments confirm the theoretical findings of Theorem~\ref{independent_theorem} and Corollary~\ref{dependent_theorem}:  \FEWA has comparable performance with \UCB. In particular, both algorithms have a logarithmic asymptotic behavior and for $\alpha =0.06$, the ratio between the regret of two algorithms is empirically lower than $2$.
Notice, the theoretical factor between the two upper bounds is $80$ (for $\alpha = 5$). This shows the ability of \myAlgorithm to be competitive for stochastic bandits.

\end{document}